\documentclass{article} 
\usepackage[final]{anneal_style}

\usepackage{microtype}
\usepackage{hyperref}
\usepackage{url}
\usepackage{booktabs}
\usepackage{amsmath,amssymb,amsthm}
\usepackage{mathtools}
\usepackage{algorithm}
\usepackage{algpseudocode}
\usepackage{enumitem}
\usepackage{multirow}
\usepackage{float}
\usepackage{graphicx}
\usepackage{subcaption}
\usepackage[most]{tcolorbox}
\usepackage{etoolbox}
\usepackage{xspace}
\usepackage{lineno}
\usepackage{soul}

\graphicspath{{Figures/}}

\definecolor{darkblue}{rgb}{0, 0, 0.5}
\definecolor{boxbluebg}{RGB}{245,248,252}
\definecolor{boxblueframe}{RGB}{86,112,148}
\definecolor{boxgreenbg}{RGB}{244,249,244}
\definecolor{boxgreenframe}{RGB}{74,120,89}
\definecolor{boxorangebg}{RGB}{252,247,240}
\definecolor{boxorangeframe}{RGB}{156,110,60}
\hypersetup{colorlinks=true, citecolor=darkblue, linkcolor=darkblue, urlcolor=darkblue}
\tcbset{
artifactbox/.style={enhanced,breakable,boxrule=0.5pt,arc=1.5mm,left=1mm,right=1mm,top=0.8mm,bottom=0.8mm,fonttitle=\bfseries\footnotesize,coltitle=black},
artifactblue/.style={artifactbox,colback=boxbluebg,colframe=boxblueframe},
artifactgreen/.style={artifactbox,colback=boxgreenbg,colframe=boxgreenframe},
artifactorange/.style={artifactbox,colback=boxorangebg,colframe=boxorangeframe},
}
\AtBeginEnvironment{tabular}{\scriptsize}

\newcommand{\sys}{\textsc{Anneal}\xspace}
\newcommand{\kg}{\ensuremath{\mathcal{KG}}\xspace}

\newcommand{\sone}{\ensuremath{\mathcal{S}_1}\xspace}
\newcommand{\stwo}{\ensuremath{\mathcal{S}_2}\xspace}
\newcommand{\pv}{p_{\text{viol}}\xspace}
\newcommand{\editkey}{\texttt{edit\_key}}
\newcommand{\Khist}{K_{\text{max-history}}}
\newcommand{\verifier}{\textsc{Verify}\xspace}

\theoremstyle{plain}
\newtheorem{theorem}{Theorem}

\title{ANNEAL: Adapting LLM Agents via Governed Symbolic Patch Learning\thanks{\textbf{Code and data available at:} \url{https://github.com/sbhakim/anneal-agents}}}

\author{
\fontsize{10}{12}\selectfont
\textbf{Safayat Bin Hakim}\textsuperscript{\textnormal{$\spadesuit$}} \hspace{2em}
\textbf{Keyan Guo}\textsuperscript{\textnormal{$\diamondsuit$}} \hspace{2em}
\textbf{Wenkai Tan}\textsuperscript{\textnormal{$\spadesuit$}} \hspace{2em}
\textbf{Alvaro Velasquez}\textsuperscript{\textnormal{$\clubsuit$}} \\
\fontsize{10}{12}\selectfont
\textbf{Shouhuai Xu}\textsuperscript{\textnormal{$\heartsuit$}} \hspace{2em}
\textbf{Houbing Herbert Song}\textsuperscript{\textnormal{$\spadesuit$}} \\[0.5ex]
\fontsize{9.5}{11.5}\selectfont\normalfont
\textsuperscript{$\spadesuit$}University of Maryland, Baltimore County \hspace{2em}
\textsuperscript{$\diamondsuit$}University at Buffalo\\
\fontsize{9.5}{11.5}\selectfont\normalfont
\textsuperscript{$\clubsuit$}University of Colorado Boulder \hspace{2em}
\textsuperscript{$\heartsuit$}University of Colorado Colorado Springs
}

\begin{document}

\ifcolmsubmission
\linenumbers
\fi

\maketitle

\begin{abstract}
LLM-based agents can recover from individual execution errors, yet they repeatedly fail on the same fault when the underlying process knowledge---operator schemas, preconditions, and constraints---remains unrepaired. Existing self-evolving approaches address this gap by updating prompts, memory, or model weights, but none directly repair the symbolic structures that encode how tasks are executed, and few provide the governance guarantees required for safe deployment. We introduce \sys, a neuro-symbolic agent that converts recurring failures into governed symbolic edits of a process knowledge graph without modifying foundation model weights. Its core mechanism, \textbf{Failure-Driven Knowledge Acquisition (FDKA)}, localizes the responsible operator, synthesizes a typed patch through constrained LLM generation, and validates the proposal via multi-dimensional scoring, symbolic guardrails, and canary testing before commit. Every accepted edit carries full provenance and deterministic rollback capability. Across four domains and 27 multi-seed runs, \sys is the only evaluated system that commits persistent structural repairs---strong baselines such as ReAct and Reflexion achieve high episodic recovery yet retain 72--100\% holdout failure rates on recurring faults, whereas \sys reduces these to 0\% in the tested recurring-failure settings. Ablation confirms that removing FDKA eliminates all structural repairs and drops success rate by up to 26.7 percentage points. These results suggest that governed symbolic repair offers a complementary paradigm to weight-level and prompt-level adaptation for persistent fault elimination.
\end{abstract}

\section{Introduction}
\label{sec:intro}

\vspace{-4mm}
\looseness=-1

When a human expert encounters a recurring process failure---say, a booking system that keeps rejecting corporate cards during blackout dates---they do not simply retry harder; they repair the underlying procedure by adding a check for blocked dates. This ability to convert recurring breakdowns into durable structural fixes is central to expertise \citep{belle2023neuro}. LLM-based agents lack this capability: they may recover within a single episode, but when the same fault reappears, they fail again because the process knowledge that caused it---operator schemas, preconditions, and constraints---remains unchanged \citep{goel2024neurosymbolic}.

Recent work on self-evolving agents addresses this gap through several paradigms. Reflection-based methods store verbal summaries across episodes \citep{shinn2023reflexion} but leave the agent's operational knowledge intact; the same structural fault can recur on every new task. Memory-augmented approaches retrieve past trajectories or distilled principles to guide decisions \citep{wu2025evolver,zhong2024memorybank}, yet they adapt what the agent \textit{recalls}, not the process definitions it \textit{executes}. Reinforcement learning on memory \citep{liu2025memrl} and evolutionary prompt optimization \citep{he2025evotest,wang2025ace} offer stronger adaptation but still operate over prompts, policies, or weights---none directly repair the typed symbolic structures that encode how tasks are planned and executed. Moreover, few of these systems provide the governance guarantees---provenance, guardrails, canary testing, rollback---that safe deployment demands \citep{wan2024towards,shao2025misevolution}.

In this work, we introduce \sys (Figure~\ref{fig:overview}), a neuro-symbolic agent that bridges this gap by converting recurring failures into governed symbolic edits of a process knowledge graph, without modifying foundation model weights. Rather than updating what the agent recalls or how it prompts, \sys repairs the operator definitions that the planner directly consults---ensuring that once a fault is fixed, it cannot structurally recur. Its core mechanism, \textbf{Failure-Driven Knowledge Acquisition (FDKA)}, treats the LLM strictly as a constrained code generator under a closed typed schema, so that all acceptance logic remains symbolic and auditable. Because no single validation step suffices for safe deployment, proposed edits must survive multi-dimensional scoring, deontic and causal guardrails, and canary testing before commit---each layer catching failure modes the others miss. A metacognitive controller further regulates when adaptation is warranted by arbitrating between fast and deliberative pathways based on uncertainty and violation signals. A formal definition of the target environments and the open-world instructability problem scope is provided in Appendix~\ref{app:problem}.

\begin{figure}[t]
\centering
\vspace{-3mm}
\includegraphics[width=\linewidth]{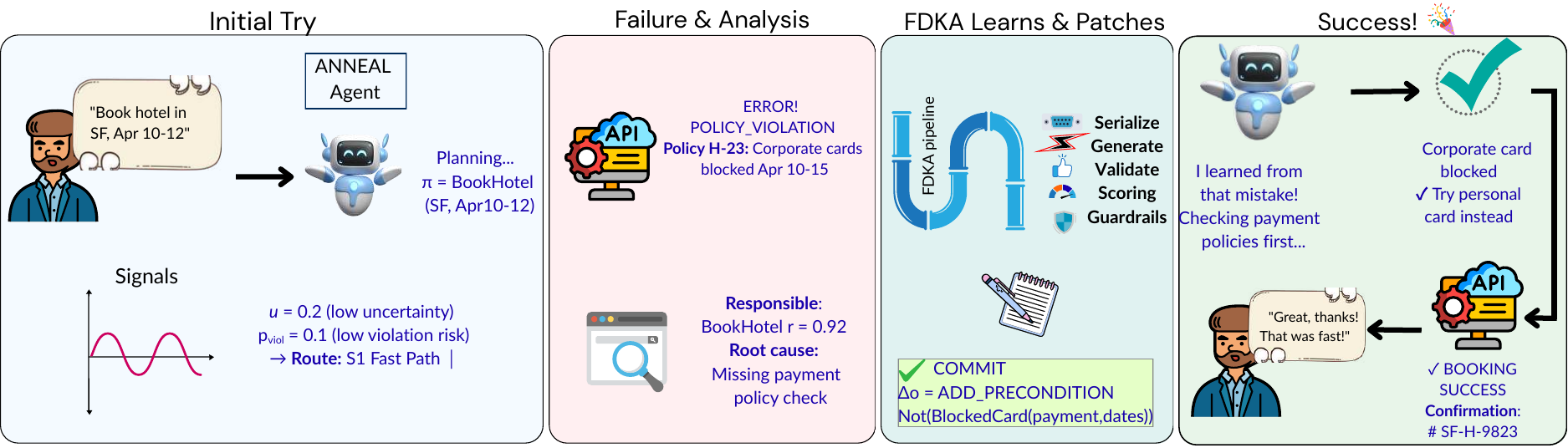}
\caption{\textbf{\sys adaptation cycle.} User requests hotel booking; execution fails due to policy violation. FDKA localizes the responsible operator, synthesizes a precondition patch, validates through scoring and guardrails, and commits. Replan with updated operator succeeds. Complete cycle: $<$2 minutes, no model retraining.}
\label{fig:overview}
\vspace{-4mm}
\end{figure}

\textbf{Contributions.} Our work combines neuro-symbolic planning \citep{goel2024neurosymbolic,kwon2025fast}, metacognitive control \citep{bergamaschi2025fast,ji2025language}, and causal/value verification \citep{jaimini2024causal,zi2025cognitive} into a governed system for symbolic process repair. We make the following contributions:

\begin{enumerate}[leftmargin=*,labelsep=4pt,itemsep=2pt]
\item \textbf{Governed failure-driven knowledge acquisition.} We propose FDKA, a pipeline that synthesizes typed symbolic edits (schema updates, precondition additions, effect refinements) under a closed JSON schema and validates them through multi-dimensional scoring, value/causal guardrails, and canary deployment. Across the 27 runs reported here, FDKA achieves 100\% patch acceptance with no observed rollbacks (Appendix~\ref{app:propose-edit},~\ref{app:governance-detail}).

\item \textbf{Metacognitive arbitration.} A stage-aware control loop monitors uncertainty and violation signals to arbitrate between fast (S1), deliberative (S2), and verify-before-act pathways under explicit budgets. Removing arbitration drops success rate by 4.7\,pp under stochastic noise, where compound policy shifts create multi-violation chains that exceed the fast path's budget.

\item \textbf{Comprehensive governance envelope.} Every committed edit carries provenance, Beta--Bernoulli trust scoring, human-in-the-loop gates, and deterministic rollback. A dedicated governance-stress benchmark confirms selective activation: all six auth-sensitive schema edits are escalated before commit in the reported suite, yielding 100\% decision accuracy on that benchmark.

\item \textbf{Empirical evidence across four domains.} Across 27 multi-seed runs, \sys is the only evaluated system that commits persistent structural repairs while achieving 94.7--100\% task success. In recurring-failure stress tests, committed patches reduce holdout failures to 0\% in the evaluated settings, whereas ReAct and Reflexion retain 72--100\% holdout failure rates despite high episodic recovery. Removing FDKA drops success rate by up to 26.7 percentage points and eliminates all structural repairs (Appendix~\ref{app:proofs} provides a TTA convergence bound).
\end{enumerate}

\noindent Section~\ref{sec:background} reviews related work; Section~\ref{sec:formalization} formalizes the agent and control loop; Section~\ref{sec:fdka} details FDKA and governance; Section~\ref{sec:empirical} presents evaluation; Section~\ref{sec:discussion} discusses limitations.

\vspace{-4mm}
\looseness=-1

\section{Related Work}
\label{sec:background}

\vspace{-4mm}
\looseness=-1

Additional background on neuro-symbolic planning, metacognitive control, and symbolic knowledge editing appears in Appendix~\ref{app:problem}; the main text focuses here on the most novelty-sensitive comparisons.

\textbf{Governance for self-evolving systems.}
Integrating causal and value reasoning is critical for safe autonomy \citep{jaimini2024causal,zi2025cognitive}. Surveys on governed deployment \citep{wan2024towards} emphasize provenance, human oversight, and rollback as prerequisites, while risk analyses show that self-evolution can degrade safety through memory, tool, model, or workflow drift \citep{shao2025misevolution}. \sys instantiates this line of work through metacognitive arbitration, multi-dimensional scoring, and an explicit governance envelope (provenance, trust, rollback).

\textbf{Self-evolving agents.}
Surveys organize self-evolving agents by what, when, and how to evolve \citep{gao2025surveyselfevolving}. Current approaches span explicit evolution over persistent artifacts \citep{lin2026aevolve}, whole-agent configuration updates at test time \citep{he2025evotest}, offline self-distillation plus online policy reinforcement \citep{wu2025evolver}, trajectory-level self-improvement \citep{lin2025seagent}, autonomous policy learning from experience \citep{sun2025seagent}, and memory-augmented or RL-based adaptation \citep{zhong2024memorybank,liu2025memrl,wang2025ace}. These systems mainly evolve prompts, memories, trajectories, or policies; the closer-in-spirit artifact-evolving family edits broader persistent artifacts (tools, configurations, workflows) but does not specifically target typed symbolic process knowledge. By contrast, \sys targets typed symbolic process knowledge---operators, preconditions, effects, tool schemas, and constraints---and accepts only governed edits that survive scoring, guardrails, and canary tests (Table~\ref{tab:selfevolving-compare}). Its contribution is therefore narrower but more concrete: governed, weight-frozen structural repair of process knowledge rather than broad self-evolution over arbitrary agent components. Adjacent model-editing work instead alters internal model parameters or auxiliary memories to change behavior, and recent results highlight nontrivial trade-offs among reliability, generalization, locality, capability retention, and safety under repeated edits \citep{li2024shouldedit,wang2024wise}. \sys differs by editing explicit symbolic process knowledge outside the foundation model, making changes directly auditable, reversible, and compatible with governance constraints.

\begin{table}[t]
\centering
\caption{\textbf{Representative self-evolving systems.} \sys is closest in spirit to artifact-level evolution, but differs in its typed symbolic repair target and stricter governed commit path.}
\label{tab:selfevolving-compare}
\begin{tabular}{lp{2.0cm}cp{2.9cm}p{4.2cm}}
\toprule
\textbf{System} & \textbf{Main target} & \textbf{Weights} & \textbf{Validation} & \textbf{Difference from \sys} \\
\midrule
A-Evolve \\ \citep{lin2026aevolve} & Persistent artifacts & Optional & Explicit verify / rollback & Broader framework over artifacts such as tools, workflows, and tests; less specifically centered on symbolic process-knowledge repair. \\
EvoTest \\ \citep{he2025evotest} & Prompt, memory, hyperparameters, tool-use routines & Frozen & Selection over evolved configs & Evolves whole-agent configuration between repeated episodes rather than committing typed symbolic patches. \\
EvolveR \\ \citep{wu2025evolver} & Distilled principles and policy & Updated & Experience curation + RL objective & Learns natural-language principles and policy behavior; does not primarily repair symbolic operators, schemas, or constraints. \\
\sys & Symbolic process knowledge & Frozen & Scoring + guardrails + canary + rollback & Commits typed edits to operators, schemas, and constraints under an explicit governance envelope. \\
\bottomrule
\end{tabular}
\vspace{-4mm}
\end{table}

The next section formalizes \sys's agent tuple and control loop before detailing FDKA and governance in Section~\ref{sec:fdka}.

\vspace{-2mm}
\looseness=-1

\section{System Overview and Control}
\label{sec:formalization}

\vspace{-4mm}
\looseness=-1

\textbf{Agent tuple.} Let $\mathcal{A} = (\Pi, \Sigma, \mathcal{E}, \mathcal{M}, \mathcal{R})$, where $\Pi$ is the HTN planner consulting PKG, $\Sigma$ is the typed symbolic state, $\mathcal{E}$ is the executor, $\mathcal{M}$ is the metacognitive controller, and $\mathcal{R} = (\mathcal{R}_{\text{rules}}, \mathcal{R}_{\text{exp}})$ contains the \textbf{rule pool} (operators, constraints) and \textbf{experience pool} (indexed failure traces). The target environments and formal scope of open-world instructability are defined in Appendix~\ref{app:problem}. Throughout, $u$ denotes uncertainty, $\pv$ violation probability, $B$ remaining budget, and $\tau_u, \tau_p$ their respective escalation thresholds.

An \textbf{operator} is $o = \langle \text{name}, \text{params}, \text{pre}(o), \text{eff}(o), \text{cost}(o) \rangle$. A \textbf{patch} $\Delta o$ modifies these fields (typically adding preconditions). Acceptance is governed by:
\begin{equation}
a = \mathbf{1}\!\left[ \textsc{Score}(\Delta o) \geq \theta \;\land\; \textsc{Verify}(\Delta o) = \texttt{allow} \right],
\label{eq:acceptance}
\end{equation}
where \textsc{Score} aggregates plausibility, consistency, utility, and risk (Appendix~\ref{app:propose-edit}); \textsc{Verify} enforces value and causal constraints. Accepted patches are staged, canaried, and committed with provenance; rejected patches are logged for audit.

\begin{figure}[t]
\centering
\vspace{-4mm}
\begin{subfigure}{\linewidth}
  \centering
  \includegraphics[width=0.87\linewidth,keepaspectratio]{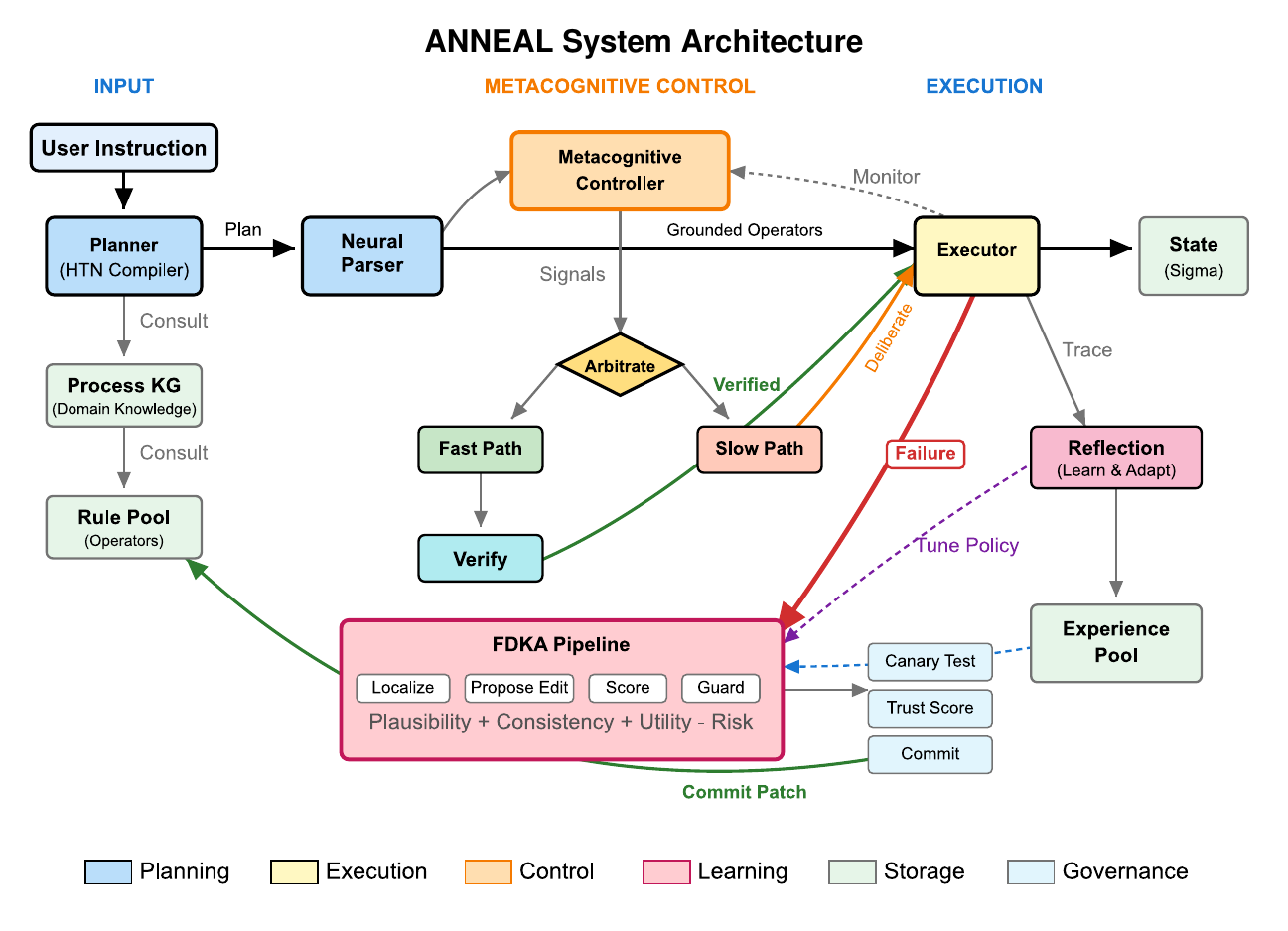}
  \vspace{-1mm}
  \caption{Execution and control loop.}
  \label{fig:architecture-loop}
\end{subfigure}

\vspace{-1mm}

\begin{subfigure}{\linewidth}
  \centering
  \includegraphics[width=\linewidth,keepaspectratio]{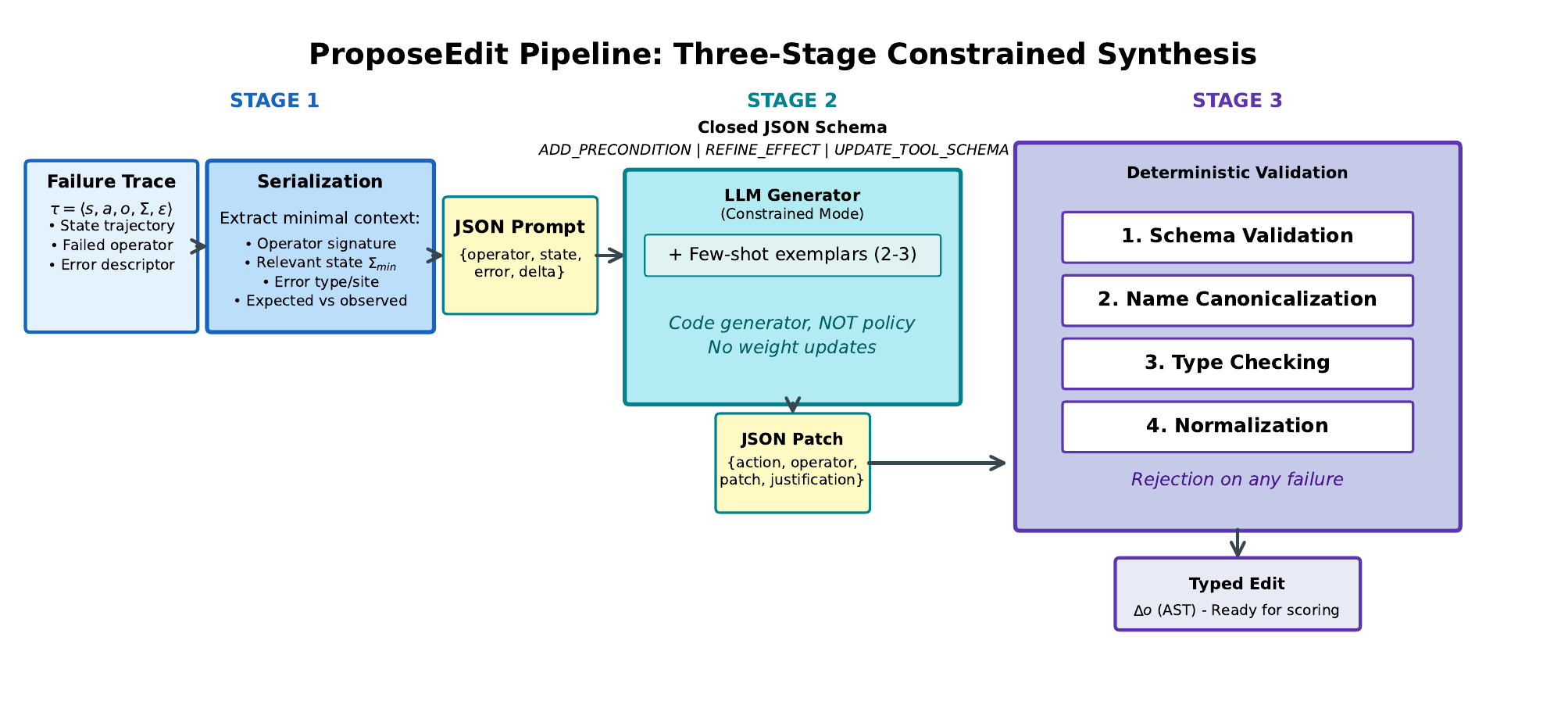}
  \vspace{-3mm}
  \caption{FDKA pipeline: localization\,$\to$\,proposal\,$\to$\,scoring\,$\to$\,guardrail\,$\to$\,canary\,$\to$\,commit.}
  \label{fig:architecture-fdka}
\end{subfigure}
\vspace{-2mm}
\caption{\textbf{\sys system architecture.} (a) Instructions compile into HTN plans via the Process Knowledge Graph (PKG); the metacognitive controller $\mathcal{M}$ monitors uncertainty $u$ (token-level entropy) and violation probability $\pv$ (logistic heuristic over precondition gaps), and arbitrates between S1 (fast), S2 (deliberative), and \verifier (precondition check) pathways under budget $B$. (b) When a failure persists, FDKA proposes a typed patch $\Delta o$ that flows through six stages---localization, proposal, scoring, guardrail, canary, commit---under governance (provenance, trust, deterministic rollback).}
\label{fig:architecture}
\vspace{-5mm}
\end{figure}

Figure~\ref{fig:architecture} illustrates the closed-loop architecture. The eight core APIs separating planning, grounding, control, adaptation, safety, and learning are listed in Appendix~\ref{app:apis}.

\textbf{Control loop.} At each planning step, the controller computes (see Appendix~\ref{app:signals} for signal computation details):
\begin{equation}
\text{pathway} = \textsc{Arb}(u, \pv, B) = \begin{cases}
\verifier & \text{if } \pv > \tau_p \;\land\; B \geq c_{\verifier} \\
\stwo & \text{if } u > \tau_u \;\land\; B \geq c_{\stwo} \\
\sone & \text{otherwise}
\end{cases}
\label{eq:arbitration}
\end{equation}
where $u \in [0,1]$ is uncertainty (grounding-completeness proxy), $\pv \in [0,1]$ is violation probability (precondition-gap heuristic), $B$ is remaining budget, and $c_{\sone} < c_{\verifier} < c_{\stwo}$ are pathway costs. Thresholds $\tau_u, \tau_p$ are tuned via reflection (exponential moving average; full reflection equations in Appendix~\ref{app:signals}).

\textbf{Verify-before-act.} When $\pv > \tau_p$, the verifier checks the next-$h$ operators ($h=3$):
\begin{equation}
\textsc{Verify}(\pi, \Sigma, h) = \begin{cases}
\texttt{allow} & \text{if } \forall i \in [1,h] : \Sigma \models \text{pre}(\pi_i) \\
\texttt{repair} & \text{if } \exists i : \Sigma \not\models \text{pre}(\pi_i) \;\land\; \text{LocalFix} \\
\texttt{block} & \text{otherwise}
\end{cases}
\label{eq:verify}
\end{equation}
Local repairs substitute trusted alternative operators or apply micro-patches from $\mathcal{R}_{\text{exp}}$, preserving plan semantics within $\leq 2$ operator swaps. LRU caching achieves $\approx 60\%$ hit rate, reducing overhead to $<$5ms per check. End-to-end latency breakdown for all pathways is in Appendix~\ref{app:complexity}.

\vspace{-3mm}
\looseness=-1

\section{Failure-Driven Knowledge Acquisition and Governance}
\label{sec:fdka}

\vspace{-3mm}
\looseness=-1

FDKA converts execution failures into verifiable edits through six stages: (1)~localization, (2)~proposal, (3)~scoring, (4)~guardrail validation, (5)~canary testing, and (6)~commit or rollback. A concrete patch example and an end-to-end walkthrough are in Appendix~\ref{app:patch-example} and Appendix~\ref{app:walkthrough}; algorithmic specification is in Appendix~\ref{app:algorithms}.

\textbf{Localization.} Given failure trace $\tau_t = \langle s_{0:t}, a_{0:t}, o, \Sigma_t, \epsilon_t \rangle$, responsibility scores operators via:
\begin{equation}
r(o' \mid \tau_t) \propto \exp(\phi(\tau_t, o')^\top w),
\label{eq:responsibility}
\end{equation}
where $\phi$ encodes symbolic deltas, tool log similarity, and parser confidence. The top-ranked operator becomes the patch target.

\textbf{Constrained generation.} \textsc{ProposeEdit} is a three-stage pipeline treating the LLM strictly as a code generator: (1)~serialize the trace to a structured JSON prompt; (2)~generate with a closed schema restricting edits to three typed categories---precondition additions, effect refinements, and tool-schema updates---at temperature $T=0.3$; (3)~deterministically parse and type-check into $\Delta o = \langle\text{scope}, \text{predicate}, \text{target}, \text{action}\rangle$. No model weights are updated; all acceptance logic is symbolic. The full three-stage pipeline with figures is in Appendix~\ref{app:propose-edit}.

\textbf{Multi-dimensional scoring.} Four dimensions, aggregated as:
\begin{equation}
\textsc{Score}(\Delta o) = w_{\text{plaus}} s_{\text{plaus}} + w_{\text{cons}} s_{\text{cons}} + w_{\text{util}} s_{\text{util}} - w_{\text{risk}} s_{\text{risk}} - \lambda_{\text{budget}},
\label{eq:score}
\end{equation}
with weights $(0.40, 0.25, 0.25, 0.10)$ and a small budget penalty $\lambda_{\text{budget}} \ge 0$ when regulation cost exceeds the configured budget. \textit{Plausibility} uses calibrated log-probability differences (Spearman $\rho = 0.83$ with human judgments). \textit{Consistency} applies dual-mode verification: fast symbolic heuristic ($<$5ms) and optional Z3 SMT check (20--50ms). \textit{Utility} counts failure prevention via counterfactual replay over $k=20$ retrieved traces. \textit{Risk} combines value-violation probability and blast radius. Individual sub-equations for all four dimensions, plus the EDCR probabilistic pre-filter, are in Appendix~\ref{app:propose-edit}.
Risk here is a soft preference term; guardrails below are hard symbolic vetoes and therefore non-redundant.

\textbf{Guardrails.} Two veto gates enforce deontic and structural constraints independently; a single veto blocks the patch regardless of score. The value guardrail queries $\kg^{\text{val}}$ (deontic rules encoded as $\langle\text{action}, \text{modality}, \text{condition}\rangle$ triples where modality $\in \{\text{Obligatory}, \text{Prohibited}, \text{Permitted}\}$):
\begin{equation}
\text{val-veto}(\Delta o) = \begin{cases}
\texttt{true} & \text{if } \exists r \in \mathcal{R}_a : r.\text{mod} = \textsc{Prohibited} \land \Delta o \text{ enables } a \\
\texttt{true} & \text{if } \exists r \in \mathcal{R}_a : r.\text{mod} = \textsc{Obligatory} \land \Delta o \text{ weakens } r.\text{cond} \\
\texttt{false} & \text{otherwise}
\end{cases}
\label{eq:val-guard}
\end{equation}
The causal guardrail uses identifiability $\iota$ and normalized impact propagation $\eta$:
\begin{equation}
\text{caus-veto}(\Delta o) = (\iota < \tau_{\text{ident}}) \lor (\eta > \tau_{\text{impact}}),
\label{eq:caus-guard}
\end{equation}
with $\tau_{\text{ident}} = 0.5$ and $\tau_{\text{impact}} = 0.6$. Full guardrail context and threshold rationale are in Appendix~\ref{app:propose-edit}.

\textbf{Staging, canary testing, and rollback.} Patches passing scoring and auto-approval enter canary testing (up to $n=8$ sandboxed runs; when fewer examples are available, a low-power mode evaluates the available examples/suite and rejects on any observed failure; strict-mode pass threshold $\geq 80\%$). This stage catches distributional mismatch that replay-based scoring may miss. Upon commit, a deterministic rollback set is computed and trust initializes via Beta--Bernoulli priors ($\rho_0 = 2/3$). Patches with $\rho < 0.3$ over 10 tasks trigger automatic rollback. The CSR canary equation, HITL gate equation, trust score equation, and commit equation are all in Appendix~\ref{app:governance-detail}. The \textbf{structural ratchet} property---once a valid precondition is installed, that failure mode becomes difficult to repeat under unchanged upstream conditions, thereby reducing recurrence within the current operator abstraction---is discussed in Appendix~\ref{app:surgical}.

\textbf{Governance pipeline.} Every committed patch carries a provenance record (source, inputs, rationale, timestamp) enabling forensic trace-back; all edits are versioned with deterministic rollback sets. A conflict-aware ledger \citep{zhang2150oneedit} detects coverage and reverse conflicts via edit-key hashing, resolving or escalating as needed. Patches must clear scoring \textit{and} guardrails \textit{and} canary testing before deployment---no single mechanism is a single point of failure. Formal equations (provenance, rollback, conflict detection, trust) are in Appendix~\ref{app:governance-detail}. In the evaluated runs, no rollbacks were observed.

\vspace{-3mm}
\looseness=-1

\section{Evaluation}
\label{sec:empirical}

\vspace{-3mm}
\looseness=-1

\subsection{Setup}
\label{sec:setup}

\vspace{-3mm}
\looseness=-1

\textbf{Domains.} \textit{Travel planning}: 25 tasks, failure rate 70\%, seed=42 (hard difficulty). \textit{Travel stochastic}: 25 tasks with transient failures and policy shifts (hard difficulty). \textit{E-commerce}: 25 tasks (order processing, payment validation, inventory management; easy difficulty as configured). \textit{ITSM}: 25 tasks (access provisioning, patch deployment, credential reset, ticketing; hard difficulty). All runs initialize with empty $\mathcal{R}_{\text{exp}}$ and 12 operators. Evaluation uses both single-seed and 3-seed protocols; each table states its horizon and seed regime explicitly. Full dataset specification is in Appendix~\ref{app:evaluation}.

\textbf{Implementation.} Patch synthesis uses GPT-4o-mini ($T=0.3$, max 512 tokens). Thresholds: $\theta=0.18$, $\tau_{\text{impact}}=0.6$, $\tau_{\text{conf}}=0.5$, $\tau_u=0.25$, $\tau_p=0.20$, budget $B=5$s. Scoring weights: plausibility 0.40, consistency 0.25, utility 0.25, risk 0.10. All headline results use OpenAI-family models via the released code; a single Anthropic Haiku 4.5 spot-check is reported at the end of \S\ref{sec:results}.

\textbf{Baselines and ablations.} \textit{Static-NS}: neuro-symbolic with fixed operators, no learning. \textit{LLM-Reflect}: LLM with cross-episode textual reflection memory (FIFO, max 20 episodes); GPT-4o-mini calls. \textit{Verify-Only}: neuro-symbolic with verification but no FDKA. \textit{ReAct} \citep{yao2023react}: per-step Thought--Action--Observation loop using GPT-4o-mini; no cross-episode memory or knowledge updates---failures recur across episodes. Ablations systematically remove Governance, Verify, Arbitration, and FDKA. In the main $-$Governance ablation we disable value and causal gates while retaining canary/rollback, so the deployment guard remains empirical; a separate governance-stress config pair fully disables governance to isolate deployment gating on auth-sensitive schema edits. Full system configurations are tabulated in Appendix~\ref{app:evaluation}.

A formal TTA convergence bound appears in Appendix~\ref{app:proofs}. In the runs reported here, observed TTA values for successfully patched failure classes ranged from 0 to 15 tasks depending on scenario and failure type.

\vspace{-3mm}
\looseness=-1

\subsection{Results}
\label{sec:results} 

\vspace{-3mm}
\looseness=-1

\textbf{Baseline comparison.} In single-seed travel planning (25 tasks, seed=42; full table in Appendix~\ref{app:baseline-table}), \sys achieves \textbf{100\% SR} with 1 accepted patch and zero terminal repeat failures. ReAct and Reflexion \citep{shinn2023reflexion} match 100\% SR via intra-episode retry but commit no structural repairs; LLM-Reflect (GPT-4o-mini) reaches 88\% with 4\% permanent failures; non-adaptive baselines plateau at 30--49\%.

\textbf{Direct multi-seed comparison.} Table~\ref{tab:multiseed} reports a controlled comparison using a common OpenAI model family across 3 agents, 3 scenarios, and 3 seeds (27 runs total). On both travel domains, all three agents reach 100\% or near-100\% success, but only \sys commits persistent symbolic repairs: 1.0$\pm$0.0 accepted patches in travel planning and travel stochastic, versus 0.0 for ReAct and Reflexion in every run. In e-commerce, Reflexion reaches the highest episodic SR at 98.7$\pm$1.9\%, but commits zero structural repairs. \sys reaches 94.7$\pm$2.3\% SR---outperforming ReAct (78.7$\pm$2.3\%)---while committing 2.7$\pm$0.6 patches per run spanning three distinct edit categories (schema updates, precondition additions, effect refinements). The stress test helps distinguish this gap: Reflexion's additional episodic SR comes from within-episode workarounds rather than persistent operator repair.

\textbf{E-commerce stress test (3 seeds).} To isolate cross-domain adaptation, we ran a controlled 14-task stress test (8 prefix + 6 holdout, all order-placement dominant) with 100\% injected tool-schema-drift failures. \sys commits one schema-update patch in every seed with TTA=0 (patch accepted at the first failure encounter), reducing holdout order-placement failures to \textbf{0\%} across all three seeds. ReAct achieves 71.4\% mean SR with 83.3\% holdout target-failure rate. Reflexion reaches 97.6$\pm$3.4\% overall SR and 94.4$\pm$7.9\% holdout SR via within-episode workarounds---it successfully completes tasks by alternative LLM-generated paths---yet still leaves 72.2$\pm$15.7\% holdout order-placement failures unresolved, consistent with the view that episodic reflection does not eliminate the underlying operator fault. Among the evaluated systems, only \sys commits a persistent repair that reduces recurrence (Figure~\ref{fig:ratchet}b).

\begin{table}[t]
\vspace{-2mm}
\centering
\small
\setlength{\tabcolsep}{4pt}
\renewcommand{\arraystretch}{0.94}
\caption{\textbf{Direct OpenAI comparison} (mean$\pm$std over 3 seeds per cell; 27 runs total). Among the evaluated systems, only \sys commits validated persistent repairs (2.7$\pm$0.6 patches in e-commerce spanning 3 distinct edit categories). Reflexion can exceed episodic SR in e-commerce but commits zero structural repairs.}
\vspace{-3mm}
\label{tab:multiseed}
\begin{tabular}{llccc}
\toprule
\textbf{Scenario} & \textbf{System} & \textbf{SR (\%)} & \textbf{RFR$_\text{term}$} & \textbf{Accepted patches} \\
\midrule
Travel planning   & \sys       & 100.0$\pm$0.0 & 0.0$\pm$0.0 & \textbf{1.0$\pm$0.0} \\
Travel planning   & ReAct      & 100.0$\pm$0.0 & 0.0$\pm$0.0 & 0.0$\pm$0.0 \\
Travel planning   & Reflexion  & 100.0$\pm$0.0 & 0.0$\pm$0.0 & 0.0$\pm$0.0 \\
\midrule
Travel stochastic & \sys       & 100.0$\pm$0.0 & 0.0$\pm$0.0 & \textbf{1.0$\pm$0.0} \\
Travel stochastic & ReAct      & \ul{97.3$\pm$2.3}  & 0.0$\pm$0.0 & 0.0$\pm$0.0 \\
Travel stochastic & Reflexion  & 100.0$\pm$0.0 & 0.0$\pm$0.0 & 0.0$\pm$0.0 \\
\midrule
E-commerce        & \sys       & \ul{94.7$\pm$2.3} & \textbf{1.3$\pm$2.3} & \textbf{2.7$\pm$0.6} \\
E-commerce        & ReAct      & 78.7$\pm$2.3  & 13.3$\pm$2.3   & 0.0$\pm$0.0 \\
E-commerce        & Reflexion  & \textbf{98.7$\pm$1.9} & 0.0$\pm$0.0   & 0.0$\pm$0.0 \\
\bottomrule
\end{tabular}
\vspace{-3mm}
\end{table}

\textbf{Recurring-failure stress test.} Table~\ref{tab:stress} isolates persistent adaptation from episodic retry. The stress suites are controlled synthetic injections designed for seeded reproducibility, not replayed production logs. We constructed a targeted 12-task travel split with 3 early API-drift exposures on the flight-booking operator, 3 filler tasks, and then the same failure class reintroduced on 6 later holdout tasks. All agents solved the holdout tasks episodically, but only \sys suppressed the target failure itself: the holdout target-failure rate falls to 0.0$\pm$0.0 after one committed patch, whereas ReAct, Reflexion, and the textual-memory baseline LLM-Reflect all remain at 100.0$\pm$0.0 despite 100.0$\pm$0.0 holdout SR --- cross-episode memory alone does not repair the underlying operator. These results are consistent with the claim that \sys modifies environment-facing process knowledge rather than merely recovering within episode (Figure~\ref{fig:ratchet}a).

\begin{table}[t]
\centering
\small
\setlength{\tabcolsep}{4pt}
\renewcommand{\arraystretch}{0.94}
\caption{\textbf{Recurring-failure stress test} (mean$\pm$std over 3 seeds). The same flight-booking API-drift class is forced in an adaptation prefix and then reintroduced on 6 holdout tasks. Holdout failure rate measures whether the target failure is still observed on those later tasks.}
\vspace{-3mm}
\label{tab:stress}
\begin{tabular}{lcccc}
\toprule
\textbf{System} & \textbf{SR (\%)} & \textbf{Holdout SR (\%)} & \textbf{Holdout fail rate} & \textbf{Accepted patches} \\
\midrule
\sys        & 100.0$\pm$0.0 & 100.0$\pm$0.0 & \textbf{0.0$\pm$0.0}   & \textbf{1.0$\pm$0.0} \\
ReAct       & \ul{94.4$\pm$9.6}  & 100.0$\pm$0.0 & 100.0$\pm$0.0 & 0.0$\pm$0.0 \\
Reflexion   & 100.0$\pm$0.0 & 100.0$\pm$0.0 & 100.0$\pm$0.0 & 0.0$\pm$0.0 \\
LLM-Reflect & 100.0$\pm$0.0 & 100.0$\pm$0.0 & 100.0$\pm$0.0 & 0.0$\pm$0.0 \\
\bottomrule
\end{tabular}
\vspace{-3mm}
\end{table}

\begin{figure}[t]
\centering
\includegraphics[width=0.85\linewidth, height=0.26\textheight, keepaspectratio]{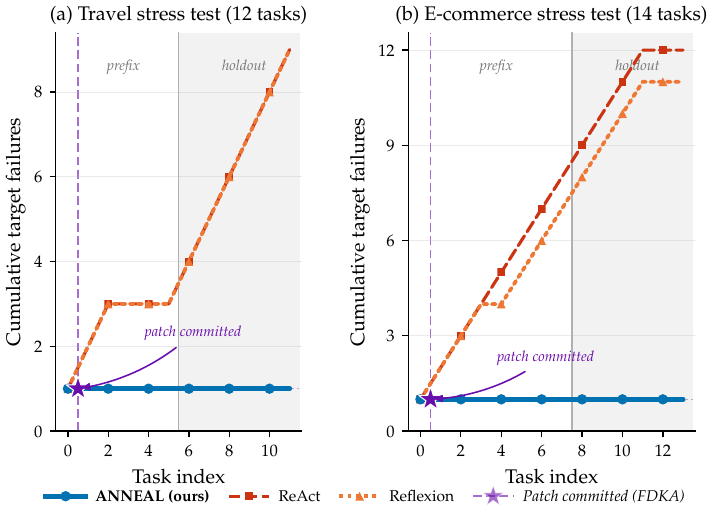}
\caption{\textbf{Adaptation curve (ratchet effect).} Cumulative target-class failures versus task index, from real per-task metrics (seed~7; all seeds identical under deterministic failure injection). (a)~Travel stress (12 tasks): \sys patches BookFlight API-drift on first encounter (TTA=0); target failures plateau at~1 while baselines accumulate 9 across prefix and holdout. (b)~E-commerce stress (14 tasks): \sys patches PlaceOrder tool-schema drift on first encounter; holdout failures remain at zero versus 11--12 for baselines. The plateau after the $\star$ marker is the signature of persistent structural repair.}
\label{fig:ratchet}
\vspace{-3mm}
\end{figure}

\textbf{Capability retention.} Because every committed patch monotonically tightens an operator, we tested whether unrelated valid behavior is preserved after patches are applied. Across three probes of increasing stringency---single cross-operator edit, single same-operator edit on valid inputs, and 2--3 cumulative heterogeneous edits across three operators---and 3 seeds, pre- and post-patch success rates are identical on every probe ($\Delta = +0.0$\,pp; 30/30 $\to$ 30/30 per condition). The probe explicitly includes patched operators on valid inputs: after an \texttt{ADD\_PRECONDITION} tightens \texttt{ApplyPromoCode}, a valid \texttt{SAVE10} promo task still succeeds, so monotonic tightening is a no-op on valid plans. Full per-probe table in Appendix~\ref{app:evaluation} (Table~\ref{tab:retention}).

\textbf{Paraphrased same-root holdout.} To test generalization beyond verbatim re-occurrence, we added a paraphrased holdout split where prefix and holdout instructions differ lexically (max token Jaccard $= 0.25$, mean $0.11$) but share the same \texttt{PlaceOrder} tool-schema-drift root cause. After one committed patch, \sys leaves \textbf{0/18} holdout target failures, versus \textbf{7/18} for Reflexion (despite 100\% holdout SR) and \textbf{11/18} for ReAct (full table in Appendix~\ref{app:evaluation}, Table~\ref{tab:variant}). This is paraphrased same-root recurrence at the instruction surface, not renamed-field or sibling-operator variation, which remain out of scope under the content-addressed \texttt{edit\_key} design (see \S\ref{sec:discussion}).

\textbf{Recovery-source decomposition and significance.} To clarify where episodic success comes from, we decompose each successful episode into three sources: (i)~\textit{no failure observed}---the operator handled the task on first pass; (ii)~\textit{post-patch recovery}---the failure class had already been repaired by a previously-committed structural edit; and (iii)~\textit{pre-patch within-episode recovery}---verify-before-act or local repair handled the first encounter without a commit. On routine 3-seed runs, the per-domain mean split is \textbf{92.9\,/\,7.1\,/\,0.0\%} (e-commerce) and \textbf{83.3\,/\,8.3\,/\,8.3\%} (travel). Bootstrap percentile tests ($B{=}10{,}000$) over the matched 3-seed paired draws give \textbf{$\Delta$SR$=+0.0$\,pp, $p=0.540$} (\sys vs.\ the strongest episodic baseline) but \textbf{$\Delta$holdout-failure-rate$=-83.3$\,pp, $p<0.001$} on the stress suites. Episode-level SR and persistent-repair effect are therefore separate axes; the structural-repair signal lives on the holdout-failure axis, not the headline-SR one. Per-source counts and 95\% bootstrap CIs are in Appendix~\ref{app:evaluation}.

\textbf{Ablation study.} Table~\ref{tab:ablation} reports component ablations across three domains (3 seeds each). In deterministic travel, removing FDKA drops SR by 4.0\,pp; other components are modular (0.0\,pp) because the failure schedule produces neither high uncertainty nor multi-step precondition chains. The stochastic variant injects compound policy shifts creating 3-violation repair chains that exceed the fast path's 2-hop budget; removing Arbitration drops SR by 4.7\,pp as the fast path exhausts its budget after two repairs while the slow path resolves all three via extended look-ahead. In e-commerce, removing FDKA drops SR by 26.7\,pp (the largest single-component impact) and eliminates all structural repairs. Removing Verify yields $+$4.0\,pp, reflecting constraint types where verify-before-act catches violations that the repair loop does not yet cover.

\begin{table}[t]
\centering
\small
\setlength{\tabcolsep}{4pt}
\renewcommand{\arraystretch}{0.94}
\caption{\textbf{Ablation study (5 conditions $\times$ 3 domains $\times$ 3 seeds each).} FDKA is essential: $-$4.0\,pp in deterministic travel; $-$26.7\,pp in e-commerce (largest single-component impact). Arbitration activates under stochastic noise ($-$4.7\,pp), where compound policy shifts create multi-violation repair chains that exceed the fast path's budget.}
\label{tab:ablation}
\begin{tabular}{lccccc}
\toprule
\textbf{Config.} & \textbf{SR (\%)} & \textbf{Patches} & \textbf{Accept.} & \textbf{$\Delta$SR} & \textbf{Reading} \\
\midrule
\multicolumn{6}{l}{\textit{Travel planning (deterministic, 3-seed mean$\pm$std, seeds 7/13/31)}} \\
\sys-Full      & 100.0{\footnotesize$\pm$0.0} & 1.0 & 100\% & ---      & Baseline \\
$-$Governance  & 100.0{\footnotesize$\pm$0.0} & 1.0 & 100\% & 0.0\,pp    & Modular \\
$-$FDKA        & 96.0{\footnotesize$\pm$0.0}  & 0   & ---   & $-$4.0\,pp & \textbf{Essential} \\
$-$Verify      & 100.0{\footnotesize$\pm$0.0} & 1.0 & 100\% & 0.0\,pp    & Modular \\
$-$Arbitration & 100.0{\footnotesize$\pm$0.0} & 1.0 & 100\% & 0.0\,pp    & Modular \\
\midrule
\multicolumn{6}{l}{\textit{Travel stochastic (25 tasks, noise-injected, 3-seed mean$\pm$std)}} \\
\sys-Full      & 100.0{\footnotesize$\pm$0.0} & 1.0 & 100\% & ---        & Baseline \\
$-$Arbitration & 95.3{\footnotesize$\pm$2.3}  & 1.0 & 100\% & $-$4.7\,pp   & \textbf{Active} \\
$-$Verify      & 100.0{\footnotesize$\pm$0.0} & 1.0 & 100\% & 0.0\,pp      & Modular \\
$-$FDKA        & 100.0{\footnotesize$\pm$0.0} & 0.0 & ---   & 0.0\,pp      & Modular \\
\midrule
\multicolumn{6}{l}{\textit{E-commerce (3-seed mean$\pm$std)}} \\
\sys-Full      & 94.7{\footnotesize$\pm$2.3} & 2.7 & 100\% & ---            & Baseline \\
$-$Governance  & 94.7{\footnotesize$\pm$2.3} & 2.7 & 100\% & 0.0\,pp          & Modular \\
$-$FDKA        & 68.0{\footnotesize$\pm$4.0} & 0.0 & ---   & $-$26.7\,pp      & \textbf{Essential} \\
$-$Verify      & 98.7{\footnotesize$\pm$2.3} & 2.7 & 100\% & $+$4.0\,pp       & Residual \\
$-$Arbitration & 94.7{\footnotesize$\pm$2.3} & 2.7 & 100\% & 0.0\,pp          & Modular \\
\bottomrule
\end{tabular}
\vspace{-4mm}
\end{table}

The ablation reveals a principled activation hierarchy. FDKA is essential across all domains because persistent knowledge gaps appear in each setting; Arbitration activates under stochastic noise where compound disruptions create multi-hop repair chains that overwhelm the fast path, and is dormant in deterministic domains where violation probability stays below the escalation threshold. In the routine benchmark, governance is intentionally quiet because accepted patches are low-risk local repairs; its role there is to certify safe commits without reducing headline SR on already-safe edits. On these routine suites, many recoveries still come from verify-before-act and local repair rather than from committed structural edits (Appendix~\ref{app:evaluation}, Table~\ref{tab:per-failure}), so the stress-test attribution should not be overgeneralized.

\textbf{Cross-domain transfer (ITSM).} Extending to a third domain under the same untuned multi-seed protocol, ITSM reaches $100.0\pm0.0\%$ SR, $100.0\pm0.0\%$ CSR, and $1.0\pm0.0$ accepted patches over 3 seeds, indicating that the current operator library transfers beyond the travel and e-commerce settings reported in Table~\ref{tab:multiseed} (full four-scenario table in Appendix~\ref{app:cross-scenario-table}).

\textbf{Governance validation.} The absence of observed rollbacks is consistent with defense-in-depth in the routine runs. To isolate governance selectivity directly, we add a dedicated e-commerce governance-stress suite: six fixed order-placement tasks all force auth-schema drift on \texttt{PlaceOrder}. Under full governance, the synthesized migration from legacy \texttt{auth\_token} to \texttt{signed\_session\_token} is escalated before commit on every task (6 proposals, 6 escalations, 0 commits, 0/6 recovered tasks). Under a governance-off stress ablation, the same repair is committed on the first task and suppresses the remaining failures (1 proposal, 1 commit, 6/6 recovered tasks). This suite is not intended to maximize average-case SR; it shows that, in this setting, security-relevant schema edits are escalated rather than silently committed, addressing the low-risk bias of the routine benchmark. A complementary synthetic 8-patch governance suite still achieves \textbf{100\% decision accuracy}; both tables appear in Appendix~\ref{app:governance-detail}; additional governance notes appear in Appendix~\ref{app:governance-notes}.

\textbf{Governance breadth.} To check that governance decisions are not specialized to a single fault class, the 8-patch suite spans four edit categories (schema migration, precondition tightening, effect refinement, value-rule violation) and four risk levels; all 8 decisions match the held-out expert label (3 commit, 3 escalate, 2 veto), with no silent commits on either value-violation case (Appendix~\ref{app:governance-detail}).

\textbf{Cross-model spot-check.} Re-running the e-commerce stress test under Anthropic Claude Haiku 4.5 as the proposal LLM (3 seeds; symbolic stack fixed) reproduces the GPT-4o-mini result exactly: SR $100.0\pm0.0\%$, holdout SR $100.0\pm0.0\%$, holdout fail rate $0.0\pm0.0\%$ (0/18), one accepted schema patch per seed at TTA$=$0, and the two models converge on the same \texttt{edit\_key}. This is a single-pair spot-check (per-seed table in Appendix~\ref{app:evaluation}), not a broad cross-model study.

\vspace{-3mm}
\looseness=-1

\section{Discussion and Limitations}
\label{sec:discussion}

\vspace{-3mm}
\looseness=-1

\textbf{Scope.} The current evidence supports a narrow claim about persistent symbolic repair, not broad general superiority. The strongest results are in recurring operator and tool-schema failures, where \sys commits validated edits and suppresses later recurrence across travel, e-commerce stress, and the untuned 3-seed ITSM extension under the released operator library; the evidence supports operator-centric transfer, not arbitrary-domain autonomy. A central observation is that high overall SR is insufficient evidence of adaptation: Reflexion reaches 97.6$\pm$3.4\% SR on the e-commerce stress test via within-episode workarounds, yet retains 72.2$\pm$15.7\% holdout target-failure rates---the underlying operator fault persists.

\textbf{Edit-space and \texttt{edit\_key} coverage.} FDKA is currently limited to operator-level edits (schema, precondition, effect); constraint-level edits, workflow restructuring, and operator synthesis for entirely novel fault classes are out of scope, and most cross-domain recoveries still come from verify-before-act and local repair rather than committed structural edits. Committed patches are keyed by a content-addressed \texttt{edit\_key} over operator name, argument schema hash, and target field: this suppresses re-proposal on identical roots (and underwrites the paraphrased same-root holdout) but does \textit{not} generalize across sibling operators with renamed fields or different schema hashes---each sibling requires its own patch, and cross-sibling transfer remains future work.

\textbf{Long-horizon coverage.} Evaluated tasks involve up-to-15-task adaptation horizons within each scenario; persistent repair under genuinely long-horizon planning---deep dependency chains, multi-day compound shifts---is not yet benchmarked, and the TTA bound (Appendix~\ref{app:proofs}) is asymptotic rather than finite-horizon.

\textbf{Component activation.} The routine benchmark is dominated by low-risk local repairs, so governance contributes primarily as deployment gating rather than as a top-line SR gain; the metacognitive arbitrator shows measurable impact only under stochastic noise ($-$4.7\,pp). Scaling the value and causal knowledge graphs, multi-provider comparisons, and richer typed edit spaces are the most important next steps; additional future directions are in Appendix~\ref{app:future}.

\vspace{-3mm}
\looseness=-1

\section{Conclusion}
\label{sec:conclusion}

\vspace{-3mm}
\looseness=-1

We presented \sys, a governed neuro-symbolic system that converts recurring agent failures into persistent symbolic repairs of process knowledge without modifying model weights. Across four domains and 27 multi-seed runs, episodic baselines retain 72--100\% holdout failure rates on recurring faults whereas \sys's committed patches reduce these to 0\%. The contribution is narrow but concrete---auditable typed symbolic repair with explicit guardrails, provenance, and rollback---and provides a foundation for richer edit spaces, multi-agent coordination, and audit-required deployment.

\section*{Ethics Statement}
\vspace{-3mm}
\looseness=-1
Self-evolving agents can introduce unsafe or non-compliant updates. \sys mitigates this risk through explicit guardrails, canary testing, provenance, rollback, and optional human review. The current experiments are sandboxed simulations and do not involve deployment on human subjects or safety-critical infrastructure.

\section*{Reproducibility Statement}
\vspace{-3mm}
\looseness=-1
Domains, baselines, model settings, thresholds, scoring weights, metrics, and fixed seeds are specified in the main paper and appendix. Code, configurations, and runnable suites for all reported experiments are released at \url{https://github.com/sbhakim/anneal-agents}.

\bibliographystyle{anneal_style}
\bibliography{references}

\appendix

\section{Problem Statement and Scope}
\label{app:problem}

We target \textit{open-world instructability}: environments where tasks are long-horizon, tools and APIs evolve (schema changes, policy updates), norms shift (new regulations, organizational flips), and novel entities appear (out-of-distribution requests, unseen edge cases). Concrete examples include enterprise scheduling under dynamic policies, multi-step travel planning with evolving booking APIs, and conversational assistants managing heterogeneous tool ecosystems.

Critically, we do not modify foundation model weights. Adaptation occurs exclusively through governed edits to the symbolic layer: operators (preconditions, effects, cost models), tool schemas, constraints, and thresholds. This design is motivated by practical constraints---continual fine-tuning risks catastrophic forgetting \citep{bell2025future}---and governance requirements: auditable, reversible changes with provenance. Our goal is to measure Time-to-Adapt (TTA)---the latency from observing a new failure class to achieving sustained improvement---as a first-class performance metric with provable bounds (Theorem~\ref{thm:tta-bound} in Appendix~\ref{app:proofs}), positioning governed symbolic repair as a viable alternative to model retraining in production settings.

\paragraph{Background notes.}
Robust instruction following benefits from neuro-symbolic decomposition: instructions compile into Hierarchical Task Networks (HTN) by consulting a Process Knowledge Graph (PKG) encoding typed operators \citep{kwon2025fast,kwon2025neuro}. Metacognitive frameworks distinguish fast heuristic reasoning (S1) from slower deliberative reasoning (S2), aggregating signals to regulate interventions within resource budgets \citep{bergamaschi2025fast,wei2024metacognitive}. Verify-before-act mechanisms use satisfiability checking to prevent infeasible plans \citep{yanfangzhou2025metagent}. However, these systems use \textit{fixed} operators, so failures recur as process knowledge remains static without repair.

Traditional continual learning updates model weights \citep{Wang2024KnowledgeEditing}, but fine-tuning is slow, expensive, and poses governance challenges \citep{bell2025future}. Neural-symbolic knowledge editing \citep{zhang2150oneedit,zhao2025clause} emphasizes stable edit keys, conflict detection, and provenance. \sys instead edits symbolic knowledge while keeping foundation model weights frozen, using the LLM as a constrained code generator rather than a learned policy.

\section{System APIs and Operational Signals}
\label{app:apis}

The architecture exposes eight core APIs enabling modular composition:
\begin{itemize}[leftmargin=*,itemsep=1pt]
\item $\Pi(\text{instruction}, \text{PKG}) \to \pi$: Compile instruction into HTN plan
\item $\textsc{Parse}(\pi, \Sigma) \to \pi_{\text{ground}}$: Ground symbolic parameters to entities
\item $\textsc{Arb}(u, \pv, B) \to \{\sone, \stwo, \verifier\}$: Arbitrate pathway under budget $B$
\item $\textsc{Verify}(\pi, \Sigma, h) \to \{\texttt{allow}, \texttt{repair}, \texttt{block}\}$: Check next-$h$ operators
\item $\textsc{ProposeEdit}(\tau, \Sigma) \to \Delta o$: Synthesize patch from failure trace
\item $\textsc{Score}(\Delta o, \mathcal{R}_{\text{exp}}) \to [0,1]$: Multi-dimensional scoring
\item $\textsc{Guardrails}(\Delta o, \kg^{\text{val}}, \kg^{\text{cau}}) \to \{\texttt{allow}, \texttt{veto}\}$: Constraint validation
\item $\textsc{Reflect}(\text{outcomes}, \tau_u, \tau_p) \to \tau'_u, \tau'_p$: Threshold tuning
\end{itemize}

Operational signals used throughout are summarized in Table~\ref{tab:signals}.

\begin{table}[h]
\centering
\caption{Operational signals for metacognitive control, arbitration, and scoring.}
\label{tab:signals}
\begin{tabular}{lll}
\toprule
\textbf{Signal} & \textbf{Range} & \textbf{Purpose} \\
\midrule
$u$ & $[0,1]$ & Uncertainty; triggers S2 escalation \\
$\pv$ & $[0,1]$ & Violation probability; triggers Verify \\
$p_{\text{gap}}$ & $[0,1]$ & Fraction of unsatisfied preconditions in next-$h$ operators \\
$i_{\text{dist}}$ & $\mathbb{R}^+$ & Min edit distance to invariant violation \\
$\rho$ & $[0,1]$ & Beta--Bernoulli trust score for patch reliability \\
\bottomrule
\end{tabular}
\end{table}

\section{Signal Computation and Arbitration}
\label{app:signals}

\textbf{Uncertainty via token-level entropy.} For sequence $y = (y_1, \ldots, y_L)$:
\begin{equation}
H_i = -\sum_{v \in \text{top-}k} p(v \mid y_{<i}, x) \log p(v \mid y_{<i}, x), \quad u = \min\!\left\{1, \frac{H(y|x)}{\bar{H}}\right\},
\end{equation}
where $\bar{H}$ is an exponential moving average over 1000 tasks ($\alpha=0.01$).

\textbf{Violation probability via logistic regression.} A 10-dimensional feature vector $d$ (precondition gap, invariant distance, tool health, novelty, recent violations, plan depth, budget slack, $u$, value impact, operator diversity) feeds a calibrated logistic regression: $\pv = \sigma(w^\top d / T_{\text{cal}})$.

\textbf{Reflection and threshold tuning.} Over the last $n=100$ tasks:
\begin{align}
\tau'_u &= (1-\alpha)\tau_u + \alpha \cdot \text{Quantile}_{0.8}(\{u_i : \text{outcome}_i = \text{failure}\}) \\
\tau'_p &= (1-\alpha)\tau_p + \alpha \cdot \text{Quantile}_{0.8}(\{p_{\text{viol},i} : \text{outcome}_i = \text{failure}\})
\end{align}
with $\alpha = 0.01$ to prevent oscillation. If $\text{SR}_{\sone} < 0.7$, thresholds increase; if $\text{SR}_{\stwo} > 0.95$, thresholds decrease.

\section{Detailed Algorithms and Complexity}
\label{app:algorithms}

\begin{algorithm}[h]
\caption{Main Control Loop with Metacognitive Arbitration}
\label{alg:main-loop}
\small
\begin{algorithmic}[1]
\Require Instruction $I$, initial state $\Sigma_0$, budget $B$, PKG, $\mathcal{R}$
\Ensure Task outcome, updated $\mathcal{R}_{\text{rules}}$
\State $\pi \gets \Pi(I, \text{PKG})$; $\pi_{\text{ground}} \gets \textsc{Parse}(\pi, \Sigma_0)$
\State $\Sigma \gets \Sigma_0$, $B_{\text{remain}} \gets B$
\While{$\pi_{\text{ground}} \neq \emptyset \wedge B_{\text{remain}} > 0$}
    \State $u \gets \textsc{Uncertainty}(\pi_{\text{ground}}, \Sigma)$; $\pv \gets \textsc{ViolationProb}(\pi_{\text{ground}}, \Sigma)$
    \State $\text{pathway} \gets \textsc{Arb}(u, \pv, B_{\text{remain}})$
    \If{$\text{pathway} = \verifier$}
        \State $\text{verdict} \gets \textsc{Verify}(\pi_{\text{ground}}, \Sigma, h=3)$
        \If{$\text{verdict} = \texttt{block}$} $\pi_{\text{ground}} \gets \Pi(I, \text{PKG})$
        \ElsIf{$\text{verdict} = \texttt{repair}$} Apply local fix; $B_{\text{remain}} \mathrel{-}= c_{\verifier}$
        \EndIf
    \ElsIf{$\text{pathway} = \stwo$}
        \State $\pi_{\text{ground}} \gets \textsc{DeliberativePlan}(I, \Sigma)$; $B_{\text{remain}} \mathrel{-}= c_{\stwo}$
    \EndIf
    \State $o_t \gets \textsc{Head}(\pi_{\text{ground}})$; $(\Sigma', \text{status}) \gets \mathcal{E}(o_t, \Sigma)$
    \If{$\text{status} = \texttt{failure}$}
        \State $\Delta o \gets \textsc{FDKA}(\tau_t, \mathcal{R})$
        \If{$\Delta o \neq \texttt{null}$} $\mathcal{R}_{\text{rules}} \mathrel{\cup}= \{\Delta o\}$; $\pi_{\text{ground}} \gets \Pi(I, \text{PKG})$
        \Else{} \Return \texttt{failure}
        \EndIf
    \Else{} $\Sigma \gets \Sigma'$; $\pi_{\text{ground}} \gets \textsc{Tail}(\pi_{\text{ground}})$
    \EndIf
    \State $B_{\text{remain}} \mathrel{-}= c_{\sone}$
\EndWhile
\State $\textsc{Reflect}(\text{trajectory}, \tau_u, \tau_p)$; \Return \texttt{success}
\end{algorithmic}
\end{algorithm}

\begin{algorithm}[h]
\caption{FDKA Pipeline: From Trace to Committed Patch}
\label{alg:fdka-pipeline}
\small
\begin{algorithmic}[1]
\Require Failure trace $\tau_t$, pools $\mathcal{R}$, KGs $\kg^{\text{val}}$, $\kg^{\text{cau}}$
\Ensure Committed patch $\Delta o$ or \texttt{null}
\State $o^* \gets \textsc{Localize}(\tau_t, \mathcal{R}_{\text{rules}})$
\State $\Delta o \gets \textsc{ProposeEdit}(\tau_t, o^*)$
\If{$\Delta o = \texttt{null}$} \Return \texttt{null} \EndIf
\State $s \gets \textsc{Score}(\Delta o, \mathcal{R}_{\text{exp}})$
\If{$s < \theta$ \textbf{or} $\textsc{val-veto}(\Delta o)$ \textbf{or} $\textsc{caus-veto}(\Delta o)$} \Return \texttt{null} \EndIf
\State $\text{conflict} \gets \texttt{Ledger.CheckAndStage}(\Delta o)$
\If{$\text{conflict} = \texttt{reverse} \wedge \rho(\Delta o) < \tau_{\text{override}}$}
    \State $\textsc{QueueHuman}(\Delta o)$; \Return \texttt{null}
\EndIf
\If{$\textsc{Gate}(\Delta o) = \texttt{queue\_human}$} $\textsc{QueueForReview}(\Delta o)$; \Return \texttt{null} \EndIf
\If{$\textsc{CanaryTest}(\Delta o, \mathcal{R}_{\text{exp}}) \geq \tau_{\text{canary}}$}
    \State $\texttt{Ledger.Commit}(\Delta o)$; \Return $\Delta o$
\Else{} $\texttt{Ledger.Rollback}(\Delta o)$; \Return \texttt{null}
\EndIf
\end{algorithmic}
\end{algorithm}

\begin{algorithm}[h]
\caption{CheckAndStage: Conflict Detection \& Resolution}
\label{alg:ledger}
\small
\begin{algorithmic}[1]
\Procedure{CheckAndStage}{$\Delta o$, Ledger}
\State $\editkey \gets \textsc{SHA256}(\text{scope} \| \text{predicate} \| \text{subject})$
\State $\text{existing} \gets \texttt{Ledger.Lookup}(\editkey)$
\If{$\text{existing} \neq \texttt{null}$}
    \If{$\text{existing.target} \neq \Delta o.\text{target}$} \Comment{Coverage conflict}
        \State \texttt{PreRollback}(\editkey); \Return \texttt{coverage\_resolved}
    \ElsIf{$\text{Negates}(\text{existing}, \Delta o)$} \Comment{Reverse conflict}
        \If{$\rho(\Delta o) \geq \tau_{\text{override}}$} \texttt{PreRollback}(\editkey); \Return \texttt{reverse\_overridden}
        \Else{} \Return \texttt{reverse\_escalate\_human}
        \EndIf
    \EndIf
\EndIf
\State \texttt{Ledger.Stage}($\editkey$, $\Delta o$, $\mathcal{R}_{\Delta o}$, $\text{prov}(\Delta o)$); \Return \texttt{ok}
\EndProcedure
\end{algorithmic}
\end{algorithm}

When edit history for a key exceeds $\Khist = 50$ entries, auto-schedule consolidation: squash to canonical patch, reset trust prior, and require fresh canary validation.

\subsection{Complexity and Latency}
\label{app:complexity}

The principal computational overheads arise from five operations:
\begin{itemize}[leftmargin=*,itemsep=2pt]
\item \textbf{Planning $\Pi$:} $\mathcal{O}(b^d)$ with branching factor $b$ and depth $d$; pruning and caching yield $d \leq 5$ typical, 10--50ms latency.
\item \textbf{Verification \verifier:} Symbolic heuristic checks complete in $<$5ms; optional Z3/SAT solving incurs 20--50ms; $\mathcal{O}(n \log n)$ in predicates per operator ($n \sim 10$); LRU caching reduces invocations by $\approx$60\%.
\item \textbf{Arbitration:} Constant time; $u$ and $\pv$ computation $<$1ms each, total $<$5ms.
\item \textbf{ProposeEdit (LLM call):} 500 input tokens, 100 output tokens; 200--500ms depending on model; dominant latency cost of FDKA.
\item \textbf{Scoring:} Probabilistic pre-filter (20ms), consistency check ($<$5ms), utility retrieval (30ms), risk classification ($<$5ms), total 60--100ms.
\end{itemize}

\textbf{End-to-end latency.} Fast path (no failure): 50--200ms for planning and execution. FDKA path: ProposeEdit ($\approx$300ms) + scoring ($\approx$100ms) + canary testing ($\approx$500--800ms for 5--8 runs, depending on available evidence) $\approx$0.9--1.2s total overhead. The metacognitive controller enforces a regulate budget $B_{\text{regulate}}$ (default 2s): if patch evaluation threatens to exceed this, scoring suspends and the patch queues for asynchronous review.

\section{Detailed ProposeEdit Pipeline and Scoring Equations}
\label{app:propose-edit}

Figure~\ref{fig:architecture}(b) illustrates the three-stage constrained-generation pipeline used inside the \textit{proposal} stage of FDKA.

\textbf{Stage 1: Input serialization.} Only predicates relevant to the failure are included ($\leq 10$ symbols): \texttt{\{operator: \{name, params, preconditions, effects\}, state\_minimal: \{relevant\_symbols\}, error: \{type, message, evidence\}\}}.

\textbf{Stage 2: Constrained generation.} The system message enforces a closed schema admitting only three edit types: precondition additions, effect refinements, and tool-schema updates. Few-shot exemplars (3--5) from $\mathcal{R}_{\text{exp}}$ demonstrate valid patches. Temperature $T=0.3$.

\textbf{Stage 3: Parsing and type checking.} A deterministic parser validates JSON structure, extracts the edit type, and type-checks against the operator schema. Invalid outputs are rejected immediately. Valid patches normalize to $\Delta o = \langle\text{scope}, \text{predicate}, \text{target}, \text{action}\rangle$.

\subsection{Individual Scoring Sub-Equations}

\textbf{Plausibility} evaluates patch alignment with the LLM's learned distribution via calibrated log-probability differences:
\begin{equation}
s_{\text{plaus}} = \sigma\!\left( \frac{1}{|\Delta o|} \left[ \log p(\Delta o \mid \tau_t) - \log p(\texttt{null} \mid \tau_t) \right] \right),
\label{eq:plausibility}
\end{equation}
where $|\Delta o|$ is patch length in tokens and \texttt{null} is a minimal baseline edit. Length normalization prevents penalizing longer patches. Validation on 200 labeled patches yields Spearman $\rho = 0.83$ with human judgments.

\textbf{Consistency} verifies that the patched operator preserves logical integrity via dual-mode verification:
\begin{equation}
s_{\text{cons}} = \begin{cases}
1 & \text{if } \textsc{SAT}(\text{pre}(o') \land \Sigma_t \land \text{Inv}) \land \textsc{TypeCheck}(o') \\
0 & \text{otherwise}
\end{cases}
\label{eq:consistency}
\end{equation}
where $o' = o \cup \Delta o$ is the patched operator and $\text{Inv}$ is the conjunction of domain invariants. The fast symbolic heuristic ($<$5ms) runs by default; Z3 SMT solving (20--50ms) is invoked optionally for formal guarantees.

\textbf{Utility} estimates whether $\Delta o$ would have prevented past failures via counterfactual replay over $k=20$ retrieved traces:
\begin{equation}
s_{\text{util}} = \frac{n_{\text{prevent}} + 0.5 \cdot n_{\text{mitigate}}}{k},
\label{eq:utility}
\end{equation}
where $n_{\text{prevent}}$ counts fully averted failures and $n_{\text{mitigate}}$ counts partially mitigated ones. If $k < 3$ (cold-start), relax acceptance threshold or queue for human review.

\textbf{Risk} quantifies potential harm combining value-violation probability and blast radius:
\begin{equation}
s_{\text{risk}} = w_{\text{val}} \cdot q_{\text{val}} + w_{\text{blast}} \cdot b,
\label{eq:risk}
\end{equation}
with $w_{\text{val}} = 0.8$, $w_{\text{blast}} = 0.2$, where $q_{\text{val}} \in [0,1]$ is value-violation probability and $b = n_{\text{affected}} / |\mathcal{R}_{\text{rules}}|$ is the blast radius fraction. High-risk patches ($s_{\text{risk}} > 0.6$) trigger human-in-the-loop gates.

\textbf{EDCR probabilistic pre-filter.} Before full scoring, a lightweight pre-filter ensures the failure class is systematic enough to warrant a patch:
\begin{equation}
\hat{P}_{\text{err}} \geq 1 - P_\alpha,
\label{eq:edcr}
\end{equation}
where $\hat{P}_{\text{err}}$ is the empirical error rate for the failure class and $P_\alpha$ is the baseline residual (errors not preventable by symbolic edits, e.g., transient tool outages; typical $P_\alpha = 0.25$). This filter rejects patches for rare, non-systematic errors, reducing noise in the rule pool.

\textbf{Scoring geometry visualization.} Figure~\ref{fig:scoring} visualizes the four scoring axes.
\begin{figure}[h]
\centering
\includegraphics[width=0.55\linewidth]{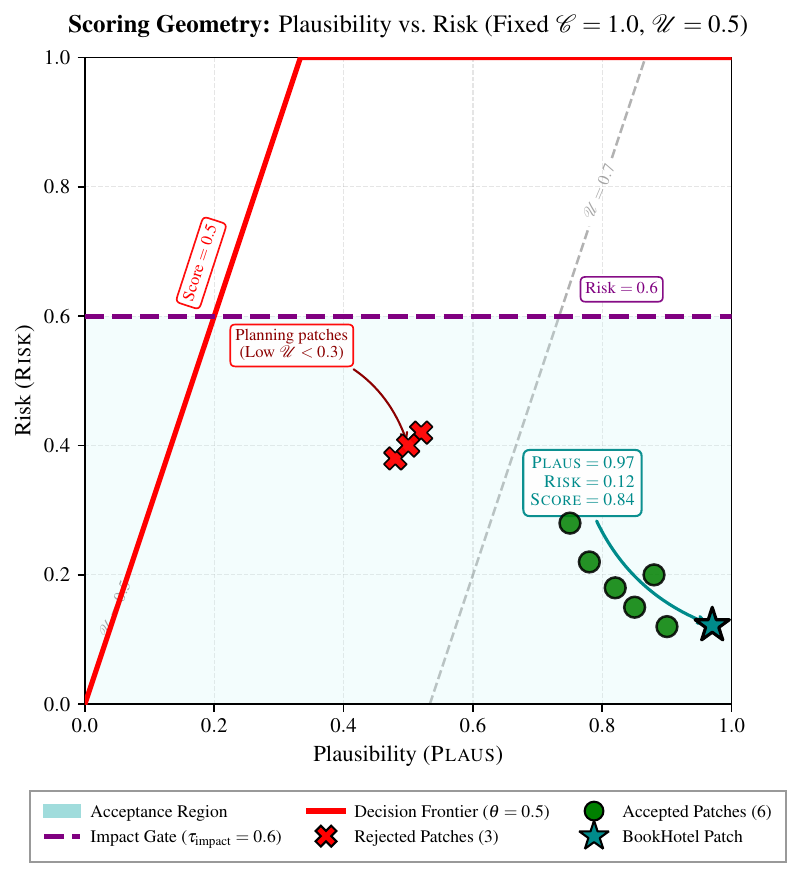}
\caption{\textbf{Four-dimensional scoring geometry.} Each patch $\Delta o$ is evaluated along plausibility (Eq.~\ref{eq:plausibility}), consistency (Eq.~\ref{eq:consistency}), utility (Eq.~\ref{eq:utility}), and risk (Eq.~\ref{eq:risk}). Aggregate score (Eq.~\ref{eq:score} in main text) must exceed $\theta = 0.18$ for acceptance.}
\label{fig:scoring}
\end{figure}

\section{Governance Details}
\label{app:governance-detail}

\subsection{Provenance, Rollback, Canary, HITL Gate, Trust Score, and Commit Equations}

\textbf{Provenance tuple.} Every committed change stores:
\begin{equation}
\text{prov}(\Delta o) = \langle \text{source}, \text{inputs}, \text{context}, \text{rationale}, t, \tau_{\text{ref}} \rangle,
\label{eq:provenance}
\end{equation}
enabling full forensic trace-back. Selective evidence deletion supports GDPR/HIPAA compliance while retaining functional patches.

\vspace{3mm}

\begin{tcolorbox}[artifactblue,title={Example Provenance / Commit Artifact}]
\small
\textbf{edit\_key:} \texttt{bookhotel.precondition.blocked\_card}\\
\textbf{source:} \texttt{FDKA / ProposeEdit}\\
\textbf{inputs:} failure trace, operator snapshot, retrieved similar traces\\
\textbf{context:} \texttt{travel\_planning}, seed 42, failure class \texttt{PAY-401}\\
\textbf{rationale:} add a blocked-card precondition so replanning avoids a policy-invalid payment path\\
\textbf{rollback set:} delete new predicate, restore prior operator version\\
\textbf{trace reference:} pointer to the originating failure trace used for forensic audit
\end{tcolorbox}

\textbf{Rollback set.} All edits are versioned:
\begin{equation}
\mathcal{R}_{\Delta o} = \{\texttt{del}(\text{new\_pred}), \texttt{add}(\text{old\_pred}), \ldots\},
\label{eq:rollback}
\end{equation}
enabling one-step deterministic reversion at any point.

\textbf{Canary Score Rate (CSR).} Canary testing executes the patched operator $o'$ on up to $n_{\text{canary}} = 8$ sandboxed scenarios:
\begin{equation}
\text{CSR} = \frac{n_{\text{pass}} + 0.5 \cdot n_{\text{mitigated}}}{n_{\text{canary}}},
\label{eq:csr}
\end{equation}
with strict-mode commit threshold $\tau_{\text{canary}} = 0.8$. When fewer retrieved examples are available, the system enters a low-power mode using the available examples and/or deterministic suite and rejects on any observed failure. Failed canary tests execute with no side effects, preventing user-facing regressions.

\textbf{Human-in-the-Loop gate.} High-risk or low-confidence patches require human approval:
\begin{equation}
\textsc{Gate}(\Delta o) = \begin{cases}
\texttt{queue\_human} & \text{if } s_{\text{risk}} > \tau_{\text{impact}} \;\lor\; \textsc{Score}(\Delta o) < \tau_{\text{conf}} \\
\texttt{auto\_approve} & \text{otherwise}
\end{cases}
\label{eq:hitl}
\end{equation}
with defaults $\tau_{\text{impact}} = 0.6$ and $\tau_{\text{conf}} = 0.5$. Safety-critical domains should use $\tau_{\text{impact}} = 0.3$, $\tau_{\text{conf}} = 0.7$. Human approvals update trust ($s \leftarrow s + 5$); denials apply penalty ($f \leftarrow f + 10$); rationales are indexed for future retrieval.

\textbf{Trust score.} Upon first commit, the patch receives a trust score via Beta--Bernoulli conjugate priors:
\begin{equation}
\rho(\Delta o) = \frac{s + \alpha}{s + f + \alpha + \beta},
\label{eq:trust}
\end{equation}
where $s$ is success count, $f$ is failure count, and $(\alpha, \beta) = (2, 1)$ is an optimistic prior ($\rho_0 = 2/3$). As the patch accrues outcomes, $\rho$ converges to empirical success rate. Patches with $\rho < 0.3$ over $n=10$ tasks are auto-flagged for rollback.

\textbf{Commit operation.} The commit stores:
\begin{equation}
\textsc{Commit}(\Delta o) = \langle \editkey, \Delta o, \mathcal{R}_{\Delta o}, \text{prov}(\Delta o), \rho_0, t \rangle,
\label{eq:commit}
\end{equation}
where $\editkey = \textsc{Hash}(\text{scope}, \text{predicate}, \text{subject})$ is a stable identifier for conflict detection, $\mathcal{R}_{\Delta o}$ is the rollback set (Eq.~\ref{eq:rollback}), $\text{prov}(\Delta o)$ is the provenance tuple (Eq.~\ref{eq:provenance}), and $t$ is the commit timestamp. To prevent thrashing after rollback, gates are temporarily tightened: $\tau_{\text{conf}} \leftarrow 1.2 \cdot \tau_{\text{conf}}$ for 50 tasks (exponential decay).

\subsection{Integrated Governance Pipeline (8-Step)}

The complete governance envelope executes as follows for each proposed patch:
\begin{enumerate}[leftmargin=*,itemsep=2pt]
\item \textbf{Propose:} FDKA generates $\Delta o$ with full provenance (source, inputs, context, rationale).
\item \textbf{Ledger check:} \textsc{CheckAndStage} detects coverage or reverse conflicts; resolves or escalates.
\item \textbf{Score:} Multi-dimensional scoring (Eq.~\ref{eq:score}) evaluates plausibility, consistency, utility, and risk.
\item \textbf{Guardrails:} Value (Eq.~\ref{eq:val-guard}) and causal (Eq.~\ref{eq:caus-guard}) gates veto or escalate unsafe edits.
\item \textbf{Human gate:} High-impact or low-confidence patches queue human review immediately after guardrails; only auto-approved patches continue.
\item \textbf{Canary test:} CSR (Eq.~\ref{eq:csr}) provides empirical safety evidence for auto-approved patches or post-review approvals.
\item \textbf{Commit:} On approval, commit (Eq.~\ref{eq:commit}) with trust initialization and rollback set.
\item \textbf{Monitor:} Trust scoring (Eq.~\ref{eq:trust}) tracks reliability; automatic rollback triggered if $\rho < 0.3$.
\end{enumerate}
This multi-layered defense ensures no single mechanism is a single point of failure.

\subsection{Governance Stress Test Results}

Table~\ref{tab:governance-stress} reports the synthetic audit; Table~\ref{tab:governance-live-stress} reports the live suite.

\begin{table}[h]
\centering
\caption{\textbf{Synthetic governance stress test outcomes.} Eight synthetic patches (4 safe, 4 unsafe) evaluated to probe guardrail behavior. All decisions match expected outcomes (100\% decision accuracy).}
\label{tab:governance-stress}
\begin{tabular}{lcc}
\toprule
\textbf{Metric} & \textbf{Count} & \textbf{Rate} \\
\midrule
Patches evaluated             & 8   & ---    \\
Value vetoes (unsafe)         & 2   & 25\%   \\
Causal escalations (ambiguous) & 2  & 25\%   \\
Canary passes (safe patches)  & 4/4 & 100\%  \\
Decision accuracy             & 8/8 & 100\%  \\
\bottomrule
\end{tabular}
\end{table}

\textbf{Live governance-activation suite.} Complementing the synthetic audit, we ran a fixed 6-task e-commerce suite that forces \texttt{auth\_schema\_drift} on \texttt{PlaceOrder} under identical seeds and task list for two configs. The same OpenAI-generated patch family is proposed: replace legacy \texttt{auth\_token} with \texttt{signed\_session\_token}. Under full governance, the migration is escalated before commit on all six tasks; under the governance-off stress ablation, the first such patch is committed and the remaining tasks recover without further proposals. This suite is used to test deployment gating rather than average-case SR.

\begin{tcolorbox}[artifactorange,title={Example Governance Escalation Record: Auth-Sensitive Schema Migration}]
\small
\textbf{Failure class:} \texttt{PlaceOrder:auth\_schema\_drift}\\
\textbf{Proposed edit:} replace legacy \texttt{auth\_token} with \texttt{signed\_session\_token} in the order-placement tool schema.\\
\textbf{Guardrail outcome:} value constraints pass, but causal review marks the migration as high-impact on the authentication path and returns \texttt{request\_human}.\\
\textbf{Canary:} not executed under full governance because escalation occurs before deployment.\\
\textbf{Counterfactual:} in the governance-off stress ablation, the first such patch is committed and the remaining five tasks recover without further proposals.
\end{tcolorbox}

\begin{table}[h]
\centering
\caption{\textbf{Live governance-activation suite.} Six fixed auth-sensitive schema-drift tasks under identical seeds/task list. Full governance escalates every proposed migration before commit; governance-off commits once and suppresses the remaining failures.}
\label{tab:governance-live-stress}
\begin{tabular}{lcccc}
\toprule
\textbf{Config} & \textbf{Patch Attempts} & \textbf{Escalations} & \textbf{Commits} & \textbf{Recovered Tasks} \\
\midrule
Full governance                & 6 & 6 & 0 & 0/6 \\
Governance-off stress ablation & 1 & 0 & 1 & 6/6 \\
\bottomrule
\end{tabular}
\end{table}

\section{TTA Convergence Proof}
\label{app:proofs}

\begin{theorem}[TTA Convergence Bound]
\label{thm:tta-bound}
Let $\mathcal{F}$ be a failure class. Assume: (i)~experience pool contains $\geq k_{\min}$ similar traces; (ii)~utility scoring detects preventable failures with probability $\geq p_{\text{detect}}$; (iii)~acceptance thresholds filter false positives at rate $\leq \epsilon$. Then with probability $\geq 1-\delta$, FDKA proposes a correct patch within:
\begin{equation}
\text{TTA} \leq \left\lceil \frac{\log(1/\delta)}{p_{\text{detect}} \cdot (1-\epsilon)} \right\rceil
\label{eq:tta-bound}
\end{equation}
tasks after observing $\mathcal{F}$.
\end{theorem}

\begin{proof}[Proof of Theorem~\ref{thm:tta-bound}]
Model adaptation as a geometric process where each task independently proposes a correct patch with probability $p_{\text{success}} = p_{\text{detect}} \cdot (1-\epsilon)$. The number of trials $T$ until first success follows $T \sim \text{Geom}(p_{\text{success}})$ with tail bound:
\[
P(T > t) = (1-p_{\text{success}})^t.
\]
Setting $P(T > t) \leq \delta$ and solving for $t$:
\[
t \geq \frac{\log(1/\delta)}{\log(1/(1-p_{\text{success}}))} \approx \frac{\log(1/\delta)}{p_{\text{success}}}
\]
for small $p_{\text{success}}$, using $\log(1/(1-p)) \approx p$. Substituting $p_{\text{success}} = p_{\text{detect}}(1-\epsilon)$ yields Eq.~\eqref{eq:tta-bound}.

\textbf{Assumption validation:} (i)~Experience pool coverage ($k_{\min}$) ensures reliable utility scoring. (ii)~$p_{\text{detect}} \geq 0.8$ from counterfactual replay validation. (iii)~Multi-layered governance achieves $\epsilon \leq 0.1$ empirically (zero rollbacks observed).
\end{proof}

\section{Evaluation Harness}
\label{app:evaluation}

\subsection{Dataset Specification}

Travel-planning benchmark: 25 tasks, failure rate 70\%, seed=42 (hard difficulty). Tasks span Policy Flip (blackout dates, approval thresholds, loyalty restrictions), Tool Drift (schema changes, rate limits, transient 503 errors), and OOD Entities (group bookings, multi-city, accessibility). Travel-stochastic adds uncertainty injection, transient failures, and policy shifts over 25 hard tasks. E-commerce contains 25 tasks in order processing, payment validation, inventory, and refund handling (easy difficulty as configured). ITSM contains 25 tasks spanning access provisioning, patch deployment, credential reset, and ticket creation (hard difficulty). Each task includes a natural language instruction, scenario class, expected failure operator, and success criterion. Separate from these 25-task suites, we report a targeted governance-activation suite: six fixed e-commerce instructions with forced \texttt{auth\_schema\_drift}, run under identical seeds/task list for full-governance and governance-off stress configs. This probe is excluded from headline SR averages because its purpose is to isolate deployment gating rather than average-case task performance. The manuscript therefore combines: (i) a single-seed travel baseline table, (ii) 3-seed direct comparisons across travel planning, travel stochastic, and e-commerce, (iii) 3-seed ablations on those same three scenarios, (iv) a 3-seed \sys cross-scenario table spanning travel planning, travel stochastic, e-commerce, and ITSM, and (v) the fixed 6-task governance-activation suite.

\subsection{Baseline and Ablation Configurations}

Table~\ref{tab:baseline-configs} details all baseline and ablation configurations.

\begin{table}[h]
\centering
\caption{\textbf{Baseline and ablation configurations.} All systems share identical planners, parsers, and domain ontologies; differences lie in adaptation mechanisms.}
\label{tab:baseline-configs}
\begin{tabular}{lp{7cm}}
\toprule
\textbf{Configuration} & \textbf{Key Properties} \\
\midrule
Static-NS             & Fixed operators, no adaptation; baseline for no-learning \\
LLM-Reflect           & LLM with cross-episode textual reflection (FIFO, max 20); GPT-4o-mini; no operator edits \\
Verify-Only           & Neuro-symbolic with precondition verification; no FDKA or governance \\
ReAct                 & Per-step Thought--Action--Observation; GPT-4o-mini; no cross-episode memory or knowledge updates \\
\sys-Full             & Complete system with all components \\
\midrule
$-$Governance (main)  & Disables value/causal gates while retaining canary and rollback; isolates routine ablation without removing empirical deployment checks \\
Governance-stress $-$Governance & Dedicated 6-task auth-schema suite disabling value/causal gates, canary, provenance, and rollback to isolate deployment gating on security-sensitive schema edits \\
$-$FDKA               & No patch synthesis; static operators post-initialization \\
$-$Verify             & No verify-before-act precondition checks \\
$-$Arbitration        & No metacognitive control; always use \sone pathway \\
\bottomrule
\end{tabular}
\end{table}

\subsection{Key Metric Formulas}

\textbf{Repeat-Failure Rate (RFR):}
\[
\text{RFR}_t = \frac{|\{i \in [t-99,t] : \text{failure}(i) \wedge \text{repeat}(i)\}|}{|\{i \in [t-99,t] : \text{failure}(i)\}|} \times 100\%,
\]
where $\text{repeat}(i) = \mathbf{1}[\exists j < i : \text{error}(j) = \text{error}(i)]$.

\textbf{Time-to-Adapt (TTA):}
\[
\text{TTA}(\mathcal{F}) = t_{\text{adapt}}(\mathcal{F}) - t_{\text{first}}(\mathcal{F}),
\]
where $t_{\text{first}}(\mathcal{F})$ is the first task containing failure class $\mathcal{F}$ and $t_{\text{adapt}}(\mathcal{F})$ is the first task at which a committed repair for $\mathcal{F}$ is installed (or, on longer horizons, the first task where task-level RFR for $\mathcal{F}$ falls below 5\%). The released code records the commit-based version in short horizons; the sustained-improvement interpretation underlies Theorem~\ref{thm:tta-bound}.

\textbf{Statistical testing:} The current manuscript mixes single-seed and 3-seed evidence. The travel baseline table is single-seed; the direct-comparison, recurring-failure stress, e-commerce stress, cross-scenario, and ablation tables report 3-seed means/std where stated; the governance-activation suite is a fixed 6-task counterfactual. \S\ref{sec:results} reports paired-draw bootstrap percentile tests ($B{=}10{,}000$) on the headline $\Delta$SR and $\Delta$holdout-failure-rate contrasts; broader formal significance testing across more domains and seeds is deferred to future work and should accompany any stronger generalization claims.

\subsection{Per-Failure Class Analysis}

Table~\ref{tab:per-failure} reports the observed learnability hierarchy.

\begin{table}[h]
\centering
\caption{\textbf{Observed hierarchy of learnability in the verified runs.} Tool-schema drift is the clearest successful patch target; several other failure classes are recovered by verification and local repair without a committed patch. Some commerce failures remain unresolved within the current operator-level edit space.}
\label{tab:per-failure}
\begin{tabular}{llcccc}
\toprule
\textbf{Failure Class} & \textbf{Tier} & \textbf{Count} & \textbf{TTA (tasks)} & \textbf{Resolved} & \textbf{Residual RFR} \\
\midrule
Travel API drift      & I & 2  & 0            & Yes & 0.0\% \\
Travel precondition gap & I & 2 & 12          & Partially & 10.0\% \\
ITSM API drift        & I & 2  & 0            & Yes & 5.0\% \\
ITSM access/MFA gates & I & 9  & 15 (local repair) & Recovered & 15.0\% \\
E-commerce order constraints & I & 14 & $\infty$ & No & 35.0\% \\
\midrule
Higher-level commerce planning & II & 2 & $\infty$ & No  & 20.0\% \\
\bottomrule
\end{tabular}
\end{table}

\subsection{Capability-Retention and Paraphrased-Holdout Tables}

Tables~\ref{tab:retention} and~\ref{tab:variant} accompany the corresponding paragraphs in \S\ref{sec:results}.

\begin{table}[h]
\centering
\small
\setlength{\tabcolsep}{4pt}
\renewcommand{\arraystretch}{0.94}
\caption{\textbf{Capability-retention probes} (10 tasks/seed $\times$ 3 seeds). Retention is exact across single and cumulative edits; the probe includes patched operators on valid inputs (e.g.\ a valid \texttt{SAVE10} promo after \texttt{ApplyPromoCode} is tightened).}
\label{tab:retention}
\begin{tabular}{lccccc}
\toprule
\textbf{Probe} & \textbf{Edits} & \textbf{Pre-SR} & \textbf{Post-SR} & \textbf{$\Delta$} & \textbf{Aggregate} \\
\midrule
Cross-operator (adjacent ops)  & 1 schema      & 100.0\% & 100.0\% & $+$0.0\,pp & 30/30 $\to$ 30/30 \\
Same-operator (valid inputs)   & 1 schema      & 100.0\% & 100.0\% & $+$0.0\,pp & 30/30 $\to$ 30/30 \\
Cumulative (multi-patch)       & 2--3, 3 cat.\ & 100.0\% & 100.0\% & $+$0.0\,pp & 30/30 $\to$ 30/30 \\
\bottomrule
\end{tabular}
\end{table}

\begin{table}[h]
\centering
\small
\setlength{\tabcolsep}{4pt}
\renewcommand{\arraystretch}{0.94}
\caption{\textbf{Paraphrased same-root holdout} (e-commerce, 3 seeds). Prefix and holdout instructions differ lexically (max token Jaccard $=0.25$, mean $0.11$) but share the same \texttt{PlaceOrder} tool-schema-drift root cause. Reflexion reaches 100\% holdout SR yet still exposes the underlying fault on 38.9\% of holdout tasks.}
\label{tab:variant}
\begin{tabular}{lcccc}
\toprule
\textbf{Agent} & \textbf{SR (\%)} & \textbf{Holdout SR (\%)} & \textbf{Holdout fail rate} & \textbf{Patches} \\
\midrule
\sys      & \textbf{100.0} & \textbf{100.0} & \textbf{0.0\,(0/18)} & \textbf{1.0$\pm$0.0} \\
Reflexion & 100.0 & 100.0 & 38.9\,(7/18)  & 0.0 \\
ReAct     & 71.4  & 88.9  & 61.1\,(11/18) & 0.0 \\
\bottomrule
\end{tabular}
\end{table}

\section{Additional Governance Notes}
\label{app:governance-notes}

The full-governance accepted patches in the current draft all pass value checks, causal checks, and canary testing, with zero observed rollbacks. The governance-off stress ablation is reported separately as a counterfactual deployment probe and should not be mixed with those full-envelope statistics. We omit model-by-model governance tables until reruns are available with matching artifact capture.

\section{Surgical Precision Analysis}
\label{app:surgical}

The system demonstrates \textbf{conservative, high-confidence patch generation} in the full-governance runs: every accepted patch clears governance and remains rollback-free (Table~\ref{tab:accepted-patches}). In e-commerce, \sys commits 2.7$\pm$0.6 patches per run spanning three distinct edit categories (schema updates, precondition additions, effect refinements), each targeting a different operator. This precision strategy indicates that FDKA triggers only when high-confidence root-cause diagnosis is achieved---not for every failure observed.

\textbf{Structural ratchet.} The versioned commit system acts as a structural ratchet: once a valid operator or tool-schema update is installed, that specific failure mode becomes difficult to repeat without an explicit rollback or a new upstream change. In the verified e-commerce runs, three distinct ratchets engage per run: a schema update on the order-placement operator (API drift), a precondition addition on the promo-code operator (field validation), and an effect refinement on the shipping-calculation operator (timeout recovery). This property contrasts with LLM-Reflect, where the same error can recur on every new task instance since no structural knowledge is updated.

\textbf{Comparison with fine-tuning.} Continual learning via fine-tuning typically requires many gradient updates per task class \citep{bell2025future}, and risks catastrophic forgetting of previously learned behaviors. \sys instead repairs \textit{structural knowledge gaps} through symbolic reasoning and governed commits. We do not make strong cost-ratio claims here because the current telemetry for token and latency accounting is incomplete in some scenarios.

\begin{table}[h]
\centering
\caption{\textbf{Representative accepted patches.} Two verified examples from travel planning and ITSM, both tool-schema repairs that passed the full governance pipeline with zero rollbacks.}
\label{tab:accepted-patches}
\small
\begin{tabular}{@{}lp{2.2cm}lp{2.45cm}p{3.3cm}@{}}
\toprule
\textbf{Scenario} & \textbf{Failure class} & \textbf{Target} & \textbf{Patch type} & \textbf{Observed effect} \\
\midrule
Travel planning & \shortstack[l]{\texttt{BookFlight:}\\\texttt{ToolError:API-V2}} & \texttt{BookFlight} & \texttt{UPDATE\_TOOL\_SCHEMA} & Canary passed; zero terminal recurrence for the accepted drift repair in the run. \\
ITSM (untuned) & \shortstack[l]{\texttt{DeployPatch:}\\\texttt{ToolError:API-V2}} & \texttt{DeployPatch} & \texttt{UPDATE\_TOOL\_SCHEMA} & Canary passed; subsequent ITSM tasks completed without terminal failure in the verified untuned run. \\
\bottomrule
\end{tabular}
\normalsize
\vspace{-4mm}
\end{table}

\vspace{-4mm}
\section{Future Directions}
\label{app:future}
\vspace{-4mm}

\textbf{Richer edit schemas and hierarchical learning.} The current proposal pipeline supports operator-level edits only. Extending to HTN-level changes (new operator templates, hierarchical decomposition rules, temporal constraints such as ``action $a$ must complete within $t$ seconds of $b$'') would address the unresolved commerce failures in Table~\ref{tab:per-failure} and increase expressiveness \citep{kwon2025fast,kwon2025neuro}. Techniques from program synthesis (sketching, version space learning) could expand the search over candidate patches while preserving deterministic parsing and typing.

\textbf{Scalable knowledge graph construction and maintenance.} LLM-assisted extraction currently addresses value constraint ($\kg^{\text{val}}$) construction. Extending to causal graphs ($\kg^{\text{cau}}$) via embedding-based completion and identifiability analysis would complete automated construction \citep{jaimini2024causal,zhao2025clause}. Value graphs should support versioning, effective dates, and jurisdictional scoping as first-class features. Temporal logic frameworks (LTL, CTL) could formalize time-dependent constraints beyond current conflict detection.

\textbf{Multi-agent federated FDKA.} Current \sys is single-agent. In multi-agent settings (multiple instances serving different users or organizations), patches learned by one agent could benefit others. Federated FDKA would aggregate patches across agents while preserving privacy: agents share anonymized provenance tuples and scoring distributions, not raw traces \citep{julian2025building}. Conflict resolution is critical: if Agent A learns ``corporate cards allowed on Mondays'' and Agent B learns ``corporate cards blocked on Mondays'' (due to different organizational policies), the federation must detect incompatibility and maintain policy-scoped patches rather than global rules.

\section{Cross-Scenario Matrix}
\label{app:cross-scenario-table}

Table~\ref{tab:scenarios} summarizes the untuned 3-seed behavior of \sys across the reported scenarios.

\begin{table}[H]
	\centering
	\caption{\textbf{Cross-scenario behavior of \sys} (mean$\pm$std over 3 seeds for untuned scenarios). Travel domains are strong, e-commerce demonstrates diverse patch types, and ITSM now also verifies as a stable untuned multi-seed transfer domain.}
	\label{tab:scenarios}
	\begin{tabular}{lccccc}
		\toprule
		\textbf{Scenario} & \textbf{SR (\%)} & \textbf{RFR$_\text{obs}$} & \textbf{RFR$_\text{term}$} & \textbf{CSR (\%)} & \textbf{Accepted patches} \\
		\midrule
		Travel plan.  & \textbf{100.0$\pm$0.0} & 89.3$\pm$22.7 & \textbf{0.0$\pm$0.0} & \textbf{100.0$\pm$0.0} & 1.0$\pm$0.0 \\
		Travel stoch. & \textbf{100.0$\pm$0.0} & 89.3$\pm$22.7 & \textbf{0.0$\pm$0.0} & \textbf{100.0$\pm$0.0} & 1.0$\pm$0.0 \\
		E-commerce    & \textbf{94.7$\pm$2.3} & 25.3$\pm$20.1 & 1.3$\pm$2.3 & 94.7$\pm$2.3 & \textbf{2.7$\pm$0.6} \\
		ITSM          & \textbf{100.0$\pm$0.0} & 52.0$\pm$0.0 & \textbf{0.0$\pm$0.0} & \textbf{100.0$\pm$0.0} & 1.0$\pm$0.0 \\
		\bottomrule
	\end{tabular}
\end{table}

\section{Single-Seed Baseline Comparison}
\label{app:baseline-table}

Table~\ref{tab:baseline} gives the full single-seed travel baseline used for the main-text comparison.

\begin{table}[H]
	\centering
	\caption{\textbf{Single-seed baseline comparison} (travel planning, 25 tasks, seed=42). \sys achieves 100\% SR via structural adaptation: zero terminal repeat failures and 1 accepted patch. Non-adaptive baselines plateau at 30--49\%; LLM-Reflect (GPT-4o-mini) achieves 88\% with 4\% permanent failures; ReAct and Reflexion match 100\% SR via intra-episode retry but cannot suppress recurrence across tasks. RFR$_\text{obs}$: repeat-failure rate during recovery; RFR$_\text{term}$: terminal failures after all recovery; $\dagger$: persistent---no cross-episode learning; $\ddagger$: cross-episode verbal buffer empty (within-episode recovery always succeeded).}
	\label{tab:baseline}

	\begin{tabular}{lcccccc}
		\toprule
		\textbf{System} & \textbf{SR (\%)} & \textbf{RFR$_\text{obs}$} & \textbf{RFR$_\text{term}$} & \textbf{CSR (\%)} & \textbf{Patches} & \textbf{Accept.} \\
		\midrule
		\sys & \textbf{100.0} & 76.0 & \textbf{0.0} & \textbf{100.0} & \textbf{1} & \textbf{100\%} \\
		\midrule
		ReAct        & 100.0 & 76.0$^\dagger$ & 0.0 & 100.0 & 0 & --- \\
		Reflexion    & 100.0 & 80.0$^\ddagger$ & 0.0 & 100.0 & 0 & --- \\
		LLM-Reflect  & 88.0  & 43.2 & 4.0  & 88.0  & 0 & --- \\
		Static-NS    & 49.2  & 65.6 & 45.6 & 49.2  & --- & --- \\
		Verify-Only  & 30.0  & 66.0 & 66.0 & 30.0  & --- & --- \\
		\bottomrule
	\end{tabular}
	
	{\scriptsize\setlength{\baselineskip}{7pt}$\dagger$~ReAct resolves each episode via intra-episode retry; failure \textit{patterns} persist unchanged across episodes (no knowledge update). $\ddagger$~Reflexion's cross-episode verbal buffer remained empty in this run (within-episode recovery always succeeded), collapsing to ReAct-equivalent behaviour.}
\end{table}

\section{FDKA Patch Example}
\label{app:patch-example}

This appendix shows one concrete governed patch from failure trace to accepted structural edit (Figure~\ref{fig:patch-example}).

\vspace{-2mm}
\begin{tcolorbox}[artifactgreen,title={Compact ProposeEdit Record: Accepted Patch Example}]
\small
\textbf{operator:} \texttt{BookHotel}\\
\textbf{edit type:} \texttt{ADD\_PRECONDITION}\\
\textbf{patch body:} \texttt{not\_blocked\_card(payment, dates)}\\
\textbf{scores:} plausibility 0.92; consistency 0.85; utility 0.89; risk 0.15; aggregate 0.87\\
\textbf{governance result:} value allow; causal allow; canary 7/7 pass\\
\textbf{commit result:} accepted with deterministic rollback capability
\end{tcolorbox}

\vspace{-2mm}
\begin{figure}[htbp]
\centering
\includegraphics[width=0.89\linewidth]{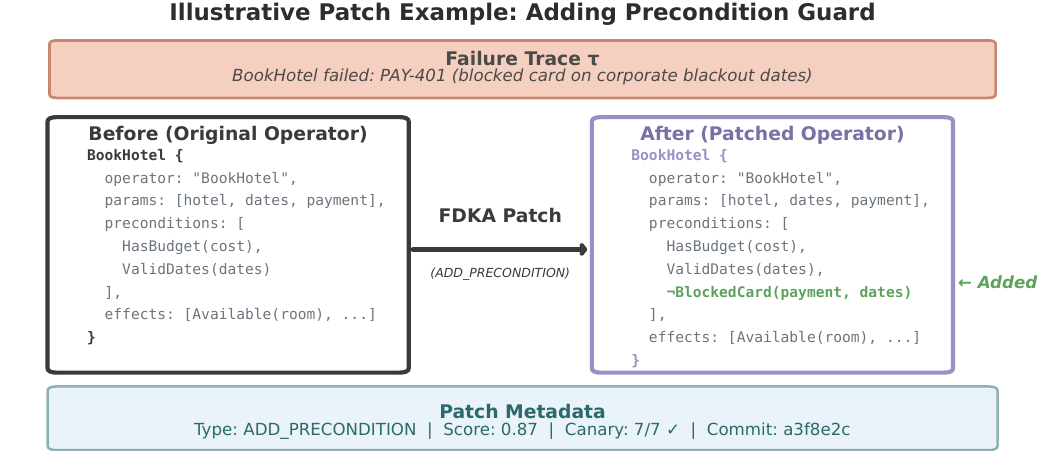}
\caption{\textbf{FDKA patch example.} A payment-authorization failure (PAY-401) on the hotel-booking operator triggers synthesis of a blocked-card precondition. The box above lists the scoring breakdown and governance outcome; the figure shows the end-to-end pipeline from failure trace to committed edit.}
\label{fig:patch-example}
\end{figure}

\section{Travel Planning Walkthrough}
\label{app:walkthrough}

\begin{tcolorbox}[artifactblue,title={Walkthrough Patch Summary: Before/After Operator Constraint}]
	\small
	\textbf{Before patch}\\
	\texttt{PRE: hotel\_available(city, dates)}\\
	\texttt{~~~~~payment\_method\_present(card)}

	\vspace{0.4em}
	\textbf{After patch}\\
	\texttt{PRE: hotel\_available(city, dates)}\\
	\texttt{~~~~~payment\_method\_present(card)}\\
	\texttt{~~~~~not\_blocked\_card(card, dates)}

	\vspace{0.4em}
	This accepted edit inserts the blocked-card check into the operator precondition, so replanning switches to a valid payment alternative instead of repeating the same failure.
\end{tcolorbox}

The full pipeline is illustrated in Figure~\ref{fig:walkthrough}. A user requests: \textit{``Book a hotel in San Francisco for April 10--12 and a flight from Newark on April 10.''}

\begin{figure}[H]
	\centering
	\includegraphics[width=\linewidth]{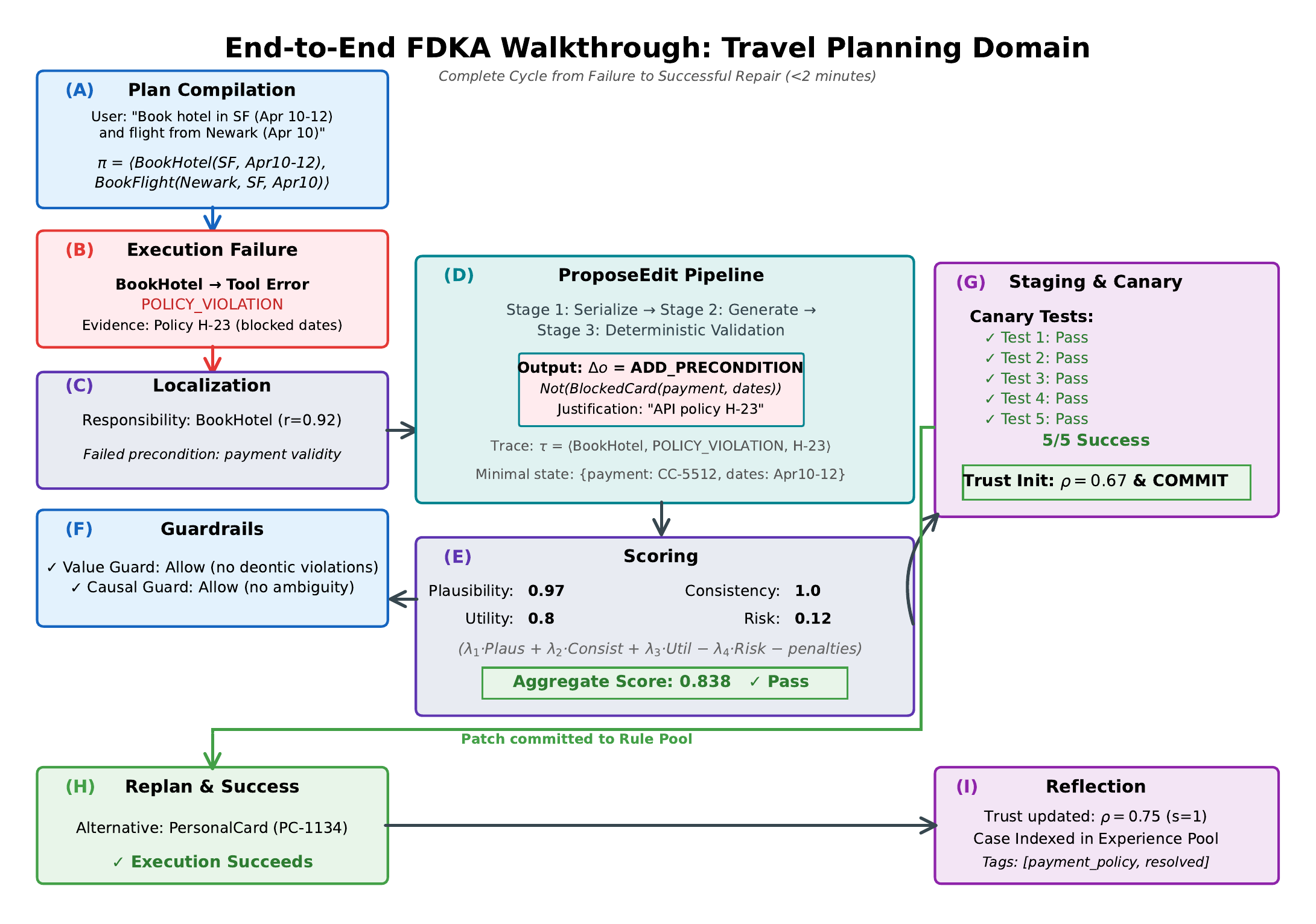}
	\caption{\textbf{End-to-end FDKA walkthrough.} (A)~Plan compilation. (B)~Policy-violation failure. (C)~Localization: hotel-booking operator ($r=0.92$). (D)~Patch proposal. (E)~Scoring yields 0.838. (F)~Guardrails allow. (G)~5/5 canary passes; commit ($\rho=0.67$). (H)~Replan succeeds. (I)~Reflection updates trust ($\rho=0.75$). Total: $<$2 minutes.}
	\label{fig:walkthrough}
\end{figure}

\textbf{Failure.} The API returns a policy violation: corporate cards are blocked for April 10--15. Trace: $\tau_t = \langle s_0, \text{BookHotel}, \Sigma_t, \text{ToolError}\rangle$.

\textbf{Localization.} The hotel-booking operator scores $r = 0.92$ (highest of 12 operators) due to high symbolic delta and tool invocation match (Eq.~\ref{eq:responsibility}).

\textbf{Scoring.} Plausibility (Eq.~\ref{eq:plausibility}): 0.97. Consistency (Eq.~\ref{eq:consistency}): 1.0 (Z3 SAT confirmed). Utility (Eq.~\ref{eq:utility}): 0.8 (16/20 traces prevented). Risk (Eq.~\ref{eq:risk}): 0.122. Aggregate (Eq.~\ref{eq:score}): \textbf{0.838}.

\textbf{Guardrails.} Value (Eq.~\ref{eq:val-guard}): no Prohibited or Obligatory violations $\to$ allowed. Causal (Eq.~\ref{eq:caus-guard}): identifiability 18/20, normalized impact 0.20 $\to$ allowed.

\textbf{Canary \& Commit.} CSR (Eq.~\ref{eq:csr}): 5/5 pass. Trust (Eq.~\ref{eq:trust}): $\rho_0 = 0.67$. Commit (Eq.~\ref{eq:commit}): edit key is a content-addressed hash of the operator scope and predicate.

\textbf{Replan \& Success.} The updated precondition (blocked-card check) is detected during replanning; a personal card alternative is selected. Execution succeeds. Trust updates to $\rho = 0.75$. \textbf{Total: 1 minute 47 seconds.}

\end{document}